\documentclass[lettersize,journal]{IEEEtran}
\usepackage{amsmath,amsfonts}
\usepackage{amssymb}
\usepackage{algorithmic}
\usepackage{algorithm}
\usepackage{array}
\usepackage[caption=false,font=normalsize,labelfont=sf,textfont=sf]{subfig}
\usepackage{textcomp}
\usepackage{stfloats}
\usepackage{url}
\usepackage{verbatim}
\usepackage{graphicx}
\usepackage{cite}
\usepackage{bbm, dsfont}
\usepackage[numbers]{natbib}
\usepackage{amsthm}

\usepackage{xcolor}
\definecolor{green(pigment)}{rgb}{0.1607, 0.3843, 0.0941}
\definecolor{blue(pigment)}{rgb}{0., 0.1484, 0.6992}
\usepackage{hyperref}
\hypersetup{hidelinks}
\usepackage{cleveref}

\newtheoremstyle{claim}
  {\topsep}
  {\topsep}
  {\itshape}
  {}
  {\itshape}
  {.}
  {.5em}
  {\thmname{#1}\thmnumber{ #2}\thmnote{ (#3)}}
\theoremstyle{claim}

\newtheorem{thm}{Theorem}[section]
\newtheorem{mydef}[thm]{Definition}
\newtheorem{mycor}[thm]{Corollary}

\newtheorem{mylem}[thm]{Lemma}
\newtheorem{myass}[thm]{Assumption}
\newtheorem{mypers}[thm]{Proposition}
\newtheorem{myrem}{Remark}

\newtheorem{innercustomthm}{Proof}
\newenvironment{custompro}[1]
  {\renewcommand\theinnercustomthm{#1}\innercustomthm}
  {\endinnercustomthm}

\usepackage{tcolorbox}
\usepackage{multirow}%
\usepackage{tabularx}
\usepackage{ctable}
\usepackage{colortbl}

\newcommand{\gr}{\rowcolor[gray]{1}}
\newcommand{\grr}{\cellcolor[gray]{1}}
  
\renewcommand{\vec}[1]{\mathbf{#1}}

\newcommand{\x}{\vec{x}}

\newcommand{\w}{\vec{w}}
\newcommand{\W}{\vec{W}}

\newcommand{\delt}{{\boldsymbol{\delta}}}

\newcommand{\D}{\mathcal{D}}

\newcommand{\U}{\vec{U}}

\newcommand{\HH}{\vec{H}}

\newcommand{\A}{\vec{A}}
\newcommand{\s}{\vec{s}}
\newcommand{\St}{\mathcal{S}}

\newcommand{\h}{\vec{h}}
\newcommand{\hh}{h^2}
\newcommand{\ac}{\vec{a}}

\newcommand{\KL}{\mathrm{KL}}
\newcommand{\g}{\vec{g}}

\newcommand{\act}{\mathrm{\vec{act}}}

\newcommand{\T}{\top}

\newcommand{\Pro}{\mathbb{P}}
\newcommand{\E}{\mathbb{E}}
\newcommand{\N}{\mathbb{N}}
\newcommand{\R}{\vec{R}}
\newcommand{\Lc}{\mathcal{L}}

\newcommand{\V}{\mathcal{V}}

\newcommand{\vecc}{\mathrm{vec}}

\newcommand{\fgm}{\mathrm{fgm}}

\newcommand{\adv}{\rm{adv}}

\newcommand{\ul}{\vec{u}}

\newcommand{\X}{\mathcal{X}_{B,d}}
\newcommand{\Y}{\mathcal{Y}}

\newcommand{\commentout}[1]{}

\newcommand{\gaojie}[1]{{\color{orange}{GJ:}{}#1}}
\newcommand{\blue}[1]{{\color{black}#1}}

\begin{document}

\title{S$^2$O: Enhancing Adversarial Training with Second-Order Statistics of Weights}

\author{Gaojie~Jin,
        Xinping~Yi,
        Wei~Huang,
        Sven~Schewe,
        \\and~Xiaowei~Huang
\thanks{The work was supported in part by the EPSRC through grants EP/X03688X/1 and EP/V026887/1, and by the National Natural Science Foundation of China under Grant 62471129.  \includegraphics[height=8pt]{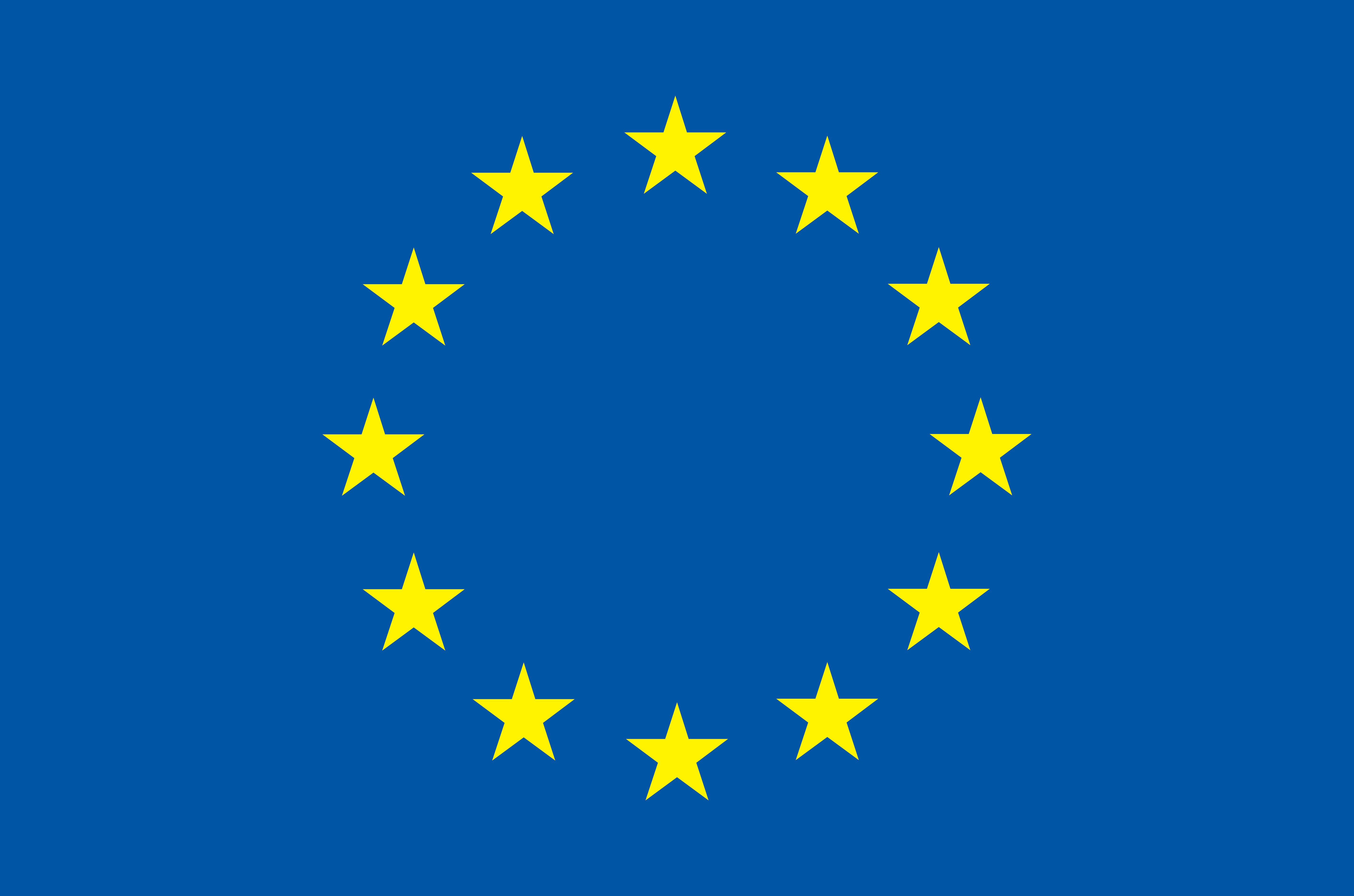} It has also received funding from the 
European Union’s Horizon Europe research and innovation action under grant 
agreement No 101212818. (Corresponding author: Xinping Yi; Xiaowei Huang)}
\thanks{Gaojie Jin is with the Department of Computer Science, University of Exeter, Exeter, UK. Email: gaojie.jin.kim@gmail.com}
\thanks{Xinping Yi is with the National Mobile Communications Research Laboratory, Southeast University, Nanjing, China. Email: xyi@seu.edu.cn}
\thanks{Wei Huang, Sven Schewe and Xiaowei Huang are with the Department of Computer Science, University of Liverpool, Liverpool, UK. Email: \{svens,xiaowei.huang\}@liverpool.ac.uk}
\thanks{
This article is an extended version of the paper~\cite{DBLP:journals/corr/abs-2203-06020}, which was presented in the IEEE/CVF conference on computer vision and pattern recognition, 2022.}
}



\maketitle

\begin{abstract}
Adversarial training has emerged as a highly effective way to improve the robustness of deep neural networks (DNNs).
It is typically conceptualized as a min-max optimization problem over model weights and adversarial perturbations, where the weights are optimized using gradient descent methods, such as SGD. 
In this paper, we propose a novel approach by treating model weights as random variables, which paves the way for enhancing adversarial training through \textbf{S}econd-Order \textbf{S}tatistics \textbf{O}ptimization (S$^2$O) over model weights.
We challenge and relax a prevalent, yet often unrealistic, assumption in prior PAC-Bayesian frameworks: the statistical independence of weights. 
From this relaxation, we derive an improved PAC-Bayesian robust generalization bound. 
Our theoretical developments suggest that optimizing the second-order statistics of weights can substantially tighten this bound. 
We complement this theoretical insight by conducting an extensive set of experiments that demonstrate that S$^2$O not only enhances the robustness and generalization of neural networks when used in isolation, but also seamlessly augments other state-of-the-art adversarial training techniques. 
The code is available at \url{https://github.com/Alexkael/S2O}.
\end{abstract}

\begin{IEEEkeywords}
DNNs, Robustness, Adversarial training, PAC-Bayes, Robust generalization bound.
\end{IEEEkeywords}

\section{Introduction}

\IEEEPARstart{D}{NN}s are known to be vulnerable to incorrect predictions with high confidence when subjected to human-imperceptible perturbations in their input, as demonstrated in several studies~\cite{goodfellow2014explaining,szegedy2013intriguing}.
This vulnerability has prompted the development of various methods to detect or mitigate such adversarial examples. 
Among these methods, adversarial training has emerged as one of the most effective strategies~\cite{papernot2016distillation,madry2017towards}.



\blue{The vanilla adversarial training is fundamentally structured as a min-max optimization problem, where the inner maximization identifies the most detrimental perturbations of training instances, while the outer minimization seeks to minimize the loss induced by these adversarial perturbations. 
In this work, we propose a novel paradigm for adversarial training by incorporating second-order statistical information of weights. 
This departure from conventional approaches is supported by both theoretical analysis and empirical evidence, demonstrating that second-order weight statistics play a crucial role in enhancing model robustness. 
Our comprehensive investigation provides rigorous theoretical guarantees and extensive experimental validation of this approach.}

Our theoretical contribution extends the PAC-Bayesian framework \cite{mcallester1999pac,dziugaite2017computing} beyond its traditional focus on model generalization and independent weights. 
Our extension, detailed in \Cref{sec:bound}, makes two fundamental contributions. 
First, we relax the conventional assumption of weight independence by explicitly modeling second-order statistics through weight correlation and normalized covariance matrices. 
Second, building upon \citet{xiao2023pac}, we expand the framework to address adversarial robustness. 
\blue{
This article significantly extends the conference paper \cite{DBLP:journals/corr/abs-2203-06020} by deriving a tighter robust generalization bound through direct model perturbation analysis. 
The enhanced bound preserves the tightness properties of the standard generalization bound while explicitly incorporating the perturbation radius $\epsilon$, yielding a more precise characterization of robust generalization.
}

This updated framework provides a theoretical indication that we can optimize the robust generalization bound during training --- and thus improve both robustness and generalization of the resulting model --- by monitoring the norms (e.g., spectral norm, determinant) of the weight correlation matrix. 
Implementing this optimization requires methods to estimate the weight correlation matrix and to guide the training process. 
As shown in \Cref{sec:4}, we employ the Laplace approximation method for estimating the correlation matrix, which is inspired by \citet{botev2017practical,ritter2018scalable}.
We then introduce a novel and efficient method named \emph{Second-Order Statistics Optimization} (S$^2$O) to enhance adversarial training.

Our extensive experimental evaluations in \Cref{sec:experiments} demonstrate that S$^2$O substantially enhances both model robustness and generalization performance. 
Beyond its effectiveness as a standalone method, S$^2$O also synergistically amplifies the performance of existing adversarial training approaches. 
\blue{Expanding upon the earlier conference version \cite{DBLP:journals/corr/abs-2203-06020} we extend, we present a comprehensive empirical study across diverse model architectures (including ViT-B and DeiT-S), multiple datasets (such as Tiny-ImageNet and Imagenette), and various robustness metrics (e.g., encompassing multiple $\ell_p$ adversarial robustness).}

\subsection{Related work}

\blue{Adversarial training is one of the most effective ways to improve robustness by exposing a model to adversarial examples.
A popular adversarial training version aims to reduce vanilla adversarial training to an equivalent or approximate expression, primarily focusing on optimizing the distance between representations of clean and adversarial examples. 
For instance, \citet{engstrom2018evaluating,kannan2018adversarial}, emphasize enforcing similarity between the logits activations of both unperturbed and adversarially perturbed versions of the same input. 
\citet{zhang2019theoretically} analyze the trade-off between robust and clean error, demonstrating bounds on their gap. 
Adversarial training can also be enhanced by pre-processing adversarial examples prior to training, for example, \citet{chen2020robust,muller2019does} modify the training process by replacing the `hard' label in adversarial instances with a `soft' label. 
This concept is further explored in \citet{zhang2020does,lee2020adversarial}, which defines a virtual sample along the adversarial direction and enriches the training distribution with soft labels. 
In addition, \citet{DBLP:journals/corr/abs-2202-10103,wang2023better} leverage extra data generated by the denoising diffusion probabilistic model (DDPM) \citep{ho2020denoising} to improve adversarial training.
\citet{wu2020adversarial} adapt the inner maximization to take one additional maximization to find a weight perturbation based on the generated adversarial examples. 
The outer minimization is then based on the perturbed weights~\citep{devries2017improved} to minimize the loss induced by the adversarial examples.
S$^2$O is different from the adversarial training methods discussed above, as it optimizes the model through weight correlation matrices.  
While the weight correlation matrix is viewed as a statistical property of weight matrices that are treated as random variables, orthogonality considers weight matrices as deterministic variables.

Recently, there are two popular test-time defense strategies against adversarial attacks: adversarial detection and adversarial purification. 
Adversarial detection enhances robustness by identifying and filtering out adversarial inputs. 
Notable methods include the Bayesian approach of \citet{deng2021libre} using likelihood ratios and the expected perturbation score approach from \citet{zhang2023detecting}. 
Adversarial purification, on the other hand, transforms adversarial examples back into clean inputs before model processing. 
Recent advances include the dynamic transformation learning \citep{li2023learning}, the score-based generative purification \citep{yoon2021adversarial}, and the diffusion model denoising approach \citep{nie2022diffusion}.
S$^2$O differs fundamentally from these inference-time defenses by enhancing model robustness during training through improved weight optimization. 
Importantly, S$^2$O can complement these test-time defense strategies, working synergistically with both adversarial detection and purification methods.

PAC-Bayes provides a powerful framework for analyzing generalization in machine learning models \citep{mcallester1999pac}, with particular strength in studying stochastic classifiers and majority vote classifiers. 
\citet{neyshabur2017pac} extend this framework to analyze generalization in deterministic DNNs, while \citet{farnia2018generalizable,xiao2023pac} further develop it to address DNN robustness. 
Our work advances this framework by relaxing the assumption of weight independence to accommodate correlated weights, yielding a novel robust generalization bound. 
This theoretical development reveals the crucial role of second-order statistics in DNN robustness, providing a rigorous foundation for our proposed S$^2$O method.}

\subsection{Motivation and contributions}
\blue{
For clarity, in this subsection, we summarize the key results from \Cref{sec:bound} and \Cref{sec:4}. 
Our analysis begins with relaxing the conventional assumption of weight independence; instead, we explicitly model correlated weights within the PAC-Bayesian framework. 
Building on this relaxed assumption, our primary theoretical contribution, formalized in \Cref{thm:advbound2} of \Cref{sec:bound}, establishes a robust generalization bound that incorporates second-order statistics of weights. 
The main results are summarized below.

\vspace{1.5mm}
\noindent
\emph{Sketch for \Cref{sec:bound}. With probability at least $1-\delta$, the inequality 
$$\text{Expected Robust Error} \leq \text{Empirical Robust Error} + \Omega_{adv}$$ 
holds within the $\ell_2$ norm data perturbation radius $\epsilon$, where $\Omega_{adv}$ is a model-dependent term characterized in part by the second-order statistics of weights.}
\vspace{1.5mm}

This theoretical result reveals that robust performance can be improved through regularizing second-order statistics of weights during adversarial training. 
Motivated by this insight, we develop a practical implementation in \Cref{sec:4} as follows.

\vspace{1.5mm}
\noindent
\emph{Sketch for \Cref{sec:4}. 
The second-order statistics of weights can be optimized using the Frobenius norm of the weight correlation matrix, i.e., $\|\R\|_F^2$.
Using Laplace approximation, we introduce a regularizer (S$^2$O) that controls the weight correlation matrix ($\|\R\|_F^2$) via the normalized covariance of the post-activation ($\|\mathbf{A}\|_F^2$).}

\begin{myrem}
This work presents the first incorporation of second-order statistics into robust generalization bounds and develops the corresponding S$^2$O regularizer based on theoretical insights. 
Detailed derivations of the robust generalization bound with second-order weight statistics are provided in \Cref{sec:bound}, while \Cref{sec:4} presents the complete development of S$^2$O. 
\end{myrem}

In \Cref{sec:experiments}, we provide an extensive set of empirical results that demonstrate the effectiveness of S$^2$O  across different datasets and network architectures.
}

\section{Preliminaries}




\subsection{Basic notation}

\blue{
Let $\St=\{(\x_1,y_1),...,(\x_m,y_m) \}$ be a training set with $m$ samples drawn i.i.d. from an underlying, fixed but unknown, distribution $\mathcal{D}$ on $\X \times \Y$. 
We suppose $\x \in \X$ and $y \in \Y$ where $\X=\{\x\in\mathbb{R}^d|\sum_{i=1}^d x_i^2\le B^2 \}$ and $\Y=\{1,2,...,n_y\}$ are the set of different labels.
We define $f_{\w}: \X \to \Y$ as the learning function parameterized by weights $\w$, where $\w$ are real-valued weights for functions mapping $\X$ to $\Y$. 
We call the set $\mathcal{H}$ of classifiers (or the set of classifier weights) the hypotheses.
Let $\ell:\mathcal{H}\times\X \times \Y \to \mathbb{R}^+$ be the loss function used in the training.
}


We consider $f_\w$ as an $n$-layer neural networks with $h$ hidden units per layer and activation functions $\act(\cdot)$.
Let $\W$, $\W_l$ be the model weight matrix and the weight matrix of $l$-th layer, respectively.
Let $\w,\w_l$ be the vectorizations of $\W,\W_l$ (i.e., $\w=\vecc(\W)$), respectively.
We can now express each $f_{\w}(\x)$ as $f_{\w}(\x) = \W_{n}\act(\W_{n-1} ...\act(\W_1 \x)...)$. 
We omit bias for convenience as bias can be incorporated into the weight matrix.
At the $l$-th layer ($l=1,\dots,n$), the latent representation before and after the activation function is denoted as $\h_l = \W_{l}\act(\W_{l-1} ...\act(\W_1 \x)...)$ and $\ac_l = \act(\W_{l} ...\act(\W_1 \x)...)$, respectively.

We use $\|\W_l\|_2$ to denote the spectral norm of $\W_l$, deﬁned as the largest singular value of $\W_l$, and $\|\W_l\|_F$, to denote the Frobenius norm of $\W_l$. 
We denote the Kronecker product by $\otimes$ and the Hadamard product by $\odot$.

\subsection{Robust margin loss}
\label{sec:loss}

The performance of a classifier $f_\w$ on the data distribution $\D$ may differ from its performance on the empirical training set $\St$.
The generalization error represents the difference between the empirical and true expected losses, estimated by training and test samples respectively. 
Following \citet{neyshabur2017pac,farnia2018generalizable}, we leverage empirical margin loss to analyze model performance within PAC-Bayes.
For any margin $\gamma>0$, we deﬁne the expected margin loss as
\begin{small}
\begin{equation}
    \Lc_{\gamma}(f_\w):=\mathop{\Pro}\limits_{(\x,y)\sim 
    \D}\Big[ f_\w(\x)[y]\le \gamma+\max\limits_{j\ne y} f_\w(\x)[j] \Big],
\end{equation}
\end{small}and let $\widehat\Lc_{\gamma}(f_\w)$ be the empirical margin loss over $\St$, i.e., 
\begin{small}
\begin{equation}
    \widehat\Lc_{\gamma}(f_\w):=\frac{1}{m}\sum_{i=1}^m \mathbbm{1}\Big[ f_\w(\x_i)[y_i]\le \gamma+\max\limits_{j\ne y_i} f_\w(\x_i)[j] \Big],
\end{equation}
\end{small}where $f_\w(\x)[j]$ denotes the $j$-th entry of $f_\w(\x)$, $\mathbbm{1}[a\le b]=1$ if $a\le b$, else $\mathbbm{1}[a\le b]=0$. 
Note that, setting $\gamma=0$ corresponds to the normal  classification loss, which will be written as $\Lc_0(f_\w)$ or $\widehat\Lc_0(f_\w)$.

\begin{myrem}
\blue{The auxiliary parameter $\gamma$ facilitates theoretical development, where $\gamma=0$ recovers the standard loss. 
Importantly, $\gamma$ only appears in the loss computed over the training set $\St$, ensuring that the classifier operates independently of the label over unseen data distribution $\D$.
That is, we leverage $\widehat\Lc_{\gamma}(f_\w)$ to bound $\Lc_{0}(f_\w)$ in this work.}
\end{myrem}

Adversarial examples are typically generated through a specific attack algorithm. 
Denote the output of such an algorithm as $\delta^{adv}_{\w}(\x)$, and $\delta^{*}_{\w}(\x)$ as the solution that maximizes the following optimization problem
\begin{small}
\begin{equation}
    \max_{\|\delta_{\w}(\x)\|_p\le \epsilon} \ell(f_\w(\x+\delta_{\w}(\x)),y).
\end{equation}
\end{small}In this context, $\ell(\cdot)$ represents the loss function, which calculates the discrepancy between the predicted and the true labels. 
We then define the robust margin loss, succinctly described in the following:
\begin{small}
\begin{equation}
\begin{aligned}
\Lc^{adv}_{\gamma}(f_\w):=\mathop{\Pro}\limits_{(\x,y)\sim 
\D}&\Big[ f_\w(\x+\delta^{*}_{\w}(\x))[y]\le \\
&\quad\quad\quad\quad\;\; \gamma+\max\limits_{j\ne y} f_\w(\x+\delta^{*}_{\w}(\x))[j] \Big].
\end{aligned}
\end{equation}
\end{small}

Adversarial training guides a model to be robust against adversarial attacks by training it with adversarially generated data, which is widely considered one of the most effective defense strategies. 
It solves the following minimax optimization problem:
\begin{small}
\begin{equation}
    \min\limits_{\w} \mathop{\E}\limits_{(\x,y)\sim \D} \Big[\ell(f_\w(\x'),y)\Big],
\end{equation}
\end{small}where $\x' = \x+\delta^*_\w(\x)$.

\subsection{PAC-Bayesian bound}

\blue{
PAC-Bayes~\cite{mcallester1999pac,thiemann2017strongly} typically provides rigorous generalization bounds for two related classifiers: the stochastic Gibbs classifier defined over a posterior distribution $Q$ on hypothesis space $\mathcal{H}$, and its deterministic counterpart, the $Q$-weighted majority vote classifier. 
The latter makes predictions by selecting the most probable output of the Gibbs classifier for each input. 
These bounds are principally characterized by the Kullback-Leibler (KL) divergence between the posterior distribution $Q$ and the prior distribution $P$ over model weights. 
Following \citet{neyshabur2017pac}, we formulate Gibbs classifiers as $f_{\w+\ul}$, where $\w$ denotes deterministic weights and $\ul$ represents a training-dependent random variable. 
\citet{neyshabur2017pac} further extend this framework by deriving generalization bounds that bridge the gap between Gibbs and deterministic classifiers.}

\begin{mylem}[\citet{neyshabur2017pac}]
\label{lem:2.1}
Let $f_{\mathbf{w}}: \X$ $\rightarrow \Y$ denote a base classifier with weights $\mathbf{w}$, and let $P$ be any prior distribution of weights that is independent of the training data.
Then, for any $\delta,\gamma>0$, and any random perturbation $\ul$ $s.t.$ $\Pro_\ul(\max_{\x}|f_{\w+\ul}(\x)-f_\w(\x)|_{\infty}<\frac{\gamma}{4})\ge\frac{1}{2}$, the following bound holds for any $f_\w$ over the training dataset of size $m$ with probability at least $1-\delta$,
\begin{small}
\begin{equation}
\begin{aligned}
    \Lc_{0}(f_\w) &\le  \widehat\Lc_{\gamma}(f_\w)+4 \sqrt{\frac{\KL(Q_{\w+\ul}\|P)+\ln \frac{6m}{\delta}}{m-1}}.
\end{aligned}
\end{equation}
\end{small}
\end{mylem}
\blue{
In this formulation, the $\KL$ divergence is computed with fixed $\w$, where the randomness stems solely from $\ul$.
The distribution of $\w + \ul$ is effectively the distribution of $\ul$ translated by $\w$. 
Following \citet{neyshabur2017pac,farnia2018generalizable}, the prior distribution $P$ over weights is selected through a systematic grid-based approach. 
Supposing $\ul\sim \mathcal{N}(0,\sigma^2\mathbf{I})$, \citet{neyshabur2017pac} derive generalization guarantees by jointly considering the sharpness limit and the Lipschitz properties of the model.
}

\begin{thm}[Standard PAC-Bayesian generalization bound, \citet{neyshabur2017pac}]
\label{thm:2.2}
Consider Lemma~\ref{lem:2.1} and $f_{\w}$ as an $n$ hidden-layer neural networks with $h$ units per layer and ReLU activation function.
Then, for any $\delta,\gamma>0$, with probability at least $1-\delta$ over a training set of size $m$, for any $\w$, we have
\begin{small}
\begin{equation}
\Lc_{0}(f_\w) \! \le \!  \widehat\Lc_{\gamma}(f_\w)+\!\mathcal{O}\! \Bigg( \sqrt{ \frac{B^2n^2h\ln(nh)\Phi(f_\w)+\ln\frac{nm}{\delta}}{\gamma^2m}} \Bigg),
\end{equation}
\end{small}where $\Phi\left(f_{\mathbf{w}}\right):=\prod_{l=1}^n\left\|\mathbf{W}_l\right\|_2^2 \sum_{l=1}^n \frac{\left\|\mathbf{W}_l\right\|_F^2}{\left\|\mathbf{W}_l\right\|_2^2}$.
\end{thm}

Denote $\widehat\Lc^{adv}_{\gamma}(f_\w)$ as the empirical estimate of the robust margin loss. 
This robust margin loss necessitates that the entire norm ball, centered around the original example $\x$, is classified correctly. 
This requirement forms the crux of norm-based adversarial robustness. 
By replacing $\delta^{*}_{\w}(\x)$ with FGM~\cite{goodfellow2014explaining}, PGM~\cite{kurakin2016adversarial} and WRM~\cite{sinha2017certifying} adversarial solution in the robust margin loss in \Cref{sec:loss}, \citet{farnia2018generalizable} develop the following theory.
\begin{thm}[Robust generalization bound, \citet{farnia2018generalizable}]
\label{thm:2.3}
Given Theorem~\ref{thm:2.2} and smooth activation functions, consider an FGM attack with noise power $\epsilon$ according to Euclidean norm $\|\cdot\|_2$.
For any $f_\w$, assume $\kappa\le \|\nabla_{\x}\ell(f_\w(\x))\|$ holds for a constant $\kappa > 0$, any $y \in \mathcal{Y}$ and any $\x$ \blue{which satisfies $\|\x\|_2\le B+\epsilon$.}
Suppose there is a constant $M \ge 1$ satisfies
\begin{small}
\begin{equation}
\forall l: \frac{1}{M} \leq \frac{\left\|\mathbf{W}_l\right\|_2}{\beta_{\mathbf{w}}} \leq M, \quad \beta_{\mathbf{w}}:=\left(\prod_{l=1}^n\left\|\mathbf{W}_i\right\|_2\right)^{1 / n}
\end{equation}
\end{small}for any $f_\w$.
Here, $\beta_\w$ represents the geometric mean of the spectral norms of $f_\w$ across all layers.
Then, for any $\delta,\gamma>0$, any $\w$ over
a training dataset of size $m$, with probability at least $1-\delta$, we have
\begin{small}
\begin{equation}
\begin{aligned}
&\Lc^{\adv}_{0}(f_\w) \le  \widehat\Lc_{\gamma}^{\adv} (f_\w)\\
&\quad\quad\;\;+\mathcal{O} \Bigg( \sqrt{\frac{(B+\epsilon)^2n^2h\ln(nh)\Phi(f_\w)+n\ln\frac{mn\ln M}{\delta}}{\gamma^2m}} \Bigg),
\end{aligned}
\end{equation}
\end{small}where $\Phi(f_\w)$ depends on the attack method. For an FGM attack, we have
\begin{small}
\begin{equation}
\begin{aligned}
    &\Phi(f_\w):=\prod_{l=1}^n \|\W_l\|_2^2 (1+C_{fgm})^2 \sum_{l=1}^n \frac{\|\W_l\|^2_F}{\|\W_l\|^2_2},\\
    &C_{fgm}=\frac{\epsilon}{\kappa}(\prod_{l=1}^n\|\W_l\|_2) \sum_{l=1}^n\prod_{j=1}^l \|\W_j\|_2.
\end{aligned}
\end{equation}
\end{small}
\end{thm}
Note that under PGM attack or WRN attack, $C_{fgm}$ can also be replaced by $C_{pgm}$ or $C_{wrm}$.

Comparing \Cref{thm:2.2} and \Cref{thm:2.3}, the main differences of the upper bounds are located in different $\Phi(f_\w)$, $B$ and $(B+\epsilon)$.
It is difficult to optimize $B+\epsilon$ in the bound as it represents the magnitudes of the adversarial examples.
\citet{xiao2023pac} find $C_{fgm}$ in $\Phi(f_\w)$ of \Cref{thm:2.3} significantly loose the bound.
To address this issue, they improve the robust generalization bound and get the following result.
\begin{thm}[Robust generalization bound, \citet{xiao2023pac}]
\label{thm:2.4}
Given \Cref{thm:2.2}, let $f_{\mathbf{w}}: \mathcal{X}_B \rightarrow \mathcal{Y}$ be an $n$ hidden-layer neural networks with $h$ units per layer and ReLU activation function.
Then, for any $\delta,\gamma>0$, with probability at least $1-\delta$ over a training set of size $m$, for any $\w$, we have
\begin{small}
\begin{equation}
\begin{aligned}
&\Lc_{0}^{adv}(f_\w) \le  \widehat\Lc_{\gamma}^{adv}(f_\w)\\
&\quad\quad\quad\quad+\mathcal{O} \Bigg( \sqrt{ \frac{(B+\epsilon)^2n^2h\ln(nh)\Phi(f_\w)+\ln\frac{nm}{\delta}}{\gamma^2m}} \Bigg),
\end{aligned}
\end{equation}
\end{small}where $\Phi\left(f_{\mathbf{w}}\right):=\prod_{l=1}^n\left\|\mathbf{W}_l\right\|_2^2 \sum_{l=1}^n \frac{\left\|\mathbf{W}_l\right\|_F^2}{\left\|\mathbf{W}_l\right\|_2^2}$.
\end{thm}

In comparing the results presented in \Cref{thm:2.3} and \Cref{thm:2.4}, it is clear that the robust generalization bound delineated in the latter is tighter, primarily due to the elimination of the additional term $C_{fgm}$. 
Consequently, in this work, we adopt the robust generalization framework as outlined in \Cref{thm:2.4}.

\section{Second-order statistics in robust generalization bound}
\label{sec:bound}

\blue{In this section, we develop a robust generalization bound based on non-spherical Gaussian random weights within the PAC-Bayesian framework.}
The assumption of random weights as a spherical Gaussian, though prevalent in various statistical and machine learning analyses \cite{neyshabur2017pac,xiao2023pac}, may not adequately capture the intricacies inherent in DNNs. 
In numerous instances, the random weights involved are interdependent, and their interrelationships play a crucial role in shaping the outcomes of our analyses. 
Consequently, it is imperative to reconsider the spherical Gaussian assumption of the posterior $Q$ (i.e., $\w+\ul$), advocating for a more adaptable PAC-Bayesian framework. 
This revised framework should account for correlations among different dimensions of $\ul$. 
Incorporating these correlations into the modeling of the posterior $Q$ allows for a more accurate representation of the complex dynamics present in the training data. 
Consequently, this approach enables the development of a PAC-Bayesian robust generalization bound that integrates a richer array of information.


\blue{
The organization of this section follows the structure illustrated in \Cref{fig:theory_flow}. 
In \Cref{sec:3.1}, we derive a robust generalization bound for non-spherical random weights (\Cref{thm:advbound}) by synthesizing three key components: the standard PAC-Bayesian bound (\Cref{thm:2.2}), the perturbed model bound (\Cref{lem:3.3}), and the random perturbation bound (\Cref{lem:3.4}), with detailed proofs provided in Appendix. 
\Cref{sec:3.2} extends this analysis to establish a refined bound (\Cref{thm:advbound2}) incorporating correlation matrices estimated from both clean and adversarial data. 
This development builds upon the correlation matrix definitions in \Cref{def:correlation matrix} and the assumptions about mixed correlation matrices in \Cref{ass:true correlation matrix}. 
Our analysis demonstrates that these correlation matrix norms---the second-order statistics---fundamentally influence the robust generalization bound.
}

\begin{figure}[t!]
\includegraphics[width=0.5
\textwidth]{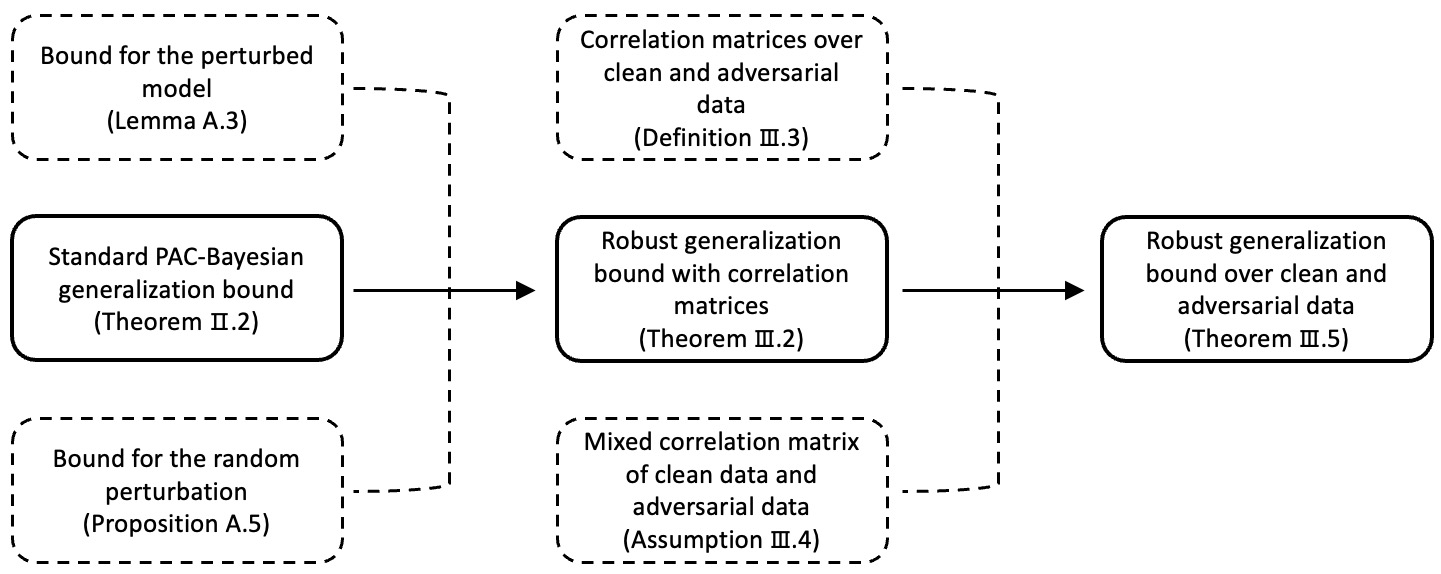}
\centering
\vspace{-5mm}
\caption{
\blue{Illustration of the theoretical framework: 
perturbation bound with consideration of correlated weights.
Under this framework, a standard generalization bound is extended to a robust generalization bound with weight correlation matrix, and further to a bound with weight correlation matrix estimated over both clean data and adversarial data.} 
}   
\vspace{-1mm}
\label{fig:theory_flow}
\end{figure}

\subsection{Robust generalization bound with correlation matrices}
\label{sec:3.1}


The adoption of a spherical Gaussian distribution for random weights has significantly streamlined the theoretical derivation of bounds in \citet{neyshabur2017pac} (refer to \Cref{thm:2.2}). 
Building on this assumption, \citet{farnia2018generalizable} formulate a robust generalization bound that takes into account the effects of adversarial attacks (as detailed in \Cref{thm:2.3}). 
Furthermore, \citet{xiao2023pac} advance this area by deriving a tighter robust generalization bound, leveraging Lipschitz continuity (see \Cref{thm:2.4}). 
However, considering the inherent complexity of neural networks, this spherical Gaussian assumption for $\ul$ may be overly restrictive, potentially overlooking critical aspects of network behavior.


\blue{In this subsection, we depart from the conventional assumption by treating $\ul$ as a non-spherical Gaussian, characterized by a correlation matrix $\R$, where generally $\R \neq \mathbf{I}$. 
This approach allows us to explore the influence of $\R$ on the established robust generalization bound. 
Before delving into the theoretical aspects, we first provide a precise definition of the correlation matrix $\R$.}

\begin{mydef}[Correlation matrix]
Let $\U_l$ be the matrixization of the random perturbation $\ul_l$, then we define
\begin{small}
\begin{equation}
\begin{aligned}
    \R_l := \frac{\E(\ul_l \ul_l^\top)}{\sigma^2},\quad 
    \R^c_l := \frac{\E(\U_l^\top \U_l)}{\sigma^2},\quad 
    \R^r_l := \frac{\E(\U_l \U_l^\top)}{\sigma^2},
\end{aligned}
\end{equation}
\end{small}where $\R_l$ is the correlation matrix of $\ul_l$, $\R^c_l$ and $\R^r_l$ are the correlation matrices of columns and rows in $\U_l$, respectively.
\end{mydef}



\blue{We now derive a margin-based robust generalization bound that incorporates weight correlation matrices. 
Our development builds upon the bound of the perturbed model from \citet{xiao2023pac}, which provides the foundation for the robust generalization bound in \Cref{thm:2.4}. 
We extend this by incorporating weight correlations into the random perturbation $\ul$, leading to a refined random perturbation bound that captures correlation-induced complexities. 
For clarity, both the bound for the perturbed model and the bound for the random perturbation are presented in the appendix. 
These developments culminate in the following PAC-Bayesian robust generalization bound.}

\begin{thm}[Robust generalization bound with correlation matrices]
\label{thm:advbound}
Given \Cref{thm:2.2}, the robust margin loss $\Lc_{\gamma}^{adv}(f_\w)$, and $f_{\mathbf{w}}: \mathcal{X}_B \rightarrow \mathcal{Y}$ as an $n$ hidden-layer neural networks with $h$ units per layer and ReLU activation function.
Let the posteriori $Q$ be over the predictors of the form $f_{\w+\ul}$, where $\ul$ is a non-spherical Gaussian with the correlation matrices $\R$, $\R^c$, and $\R^r$.
Then, for any $\delta, \gamma>0$, with probability at least $1-\delta$ over a training set of size $m$, for any $\w$, we have
\begin{small}
\begin{equation}
\begin{aligned}
&\Lc^{adv}_{0}(f_\w) \le  \widehat\Lc_{\gamma}^{adv}(f_\w)\\
&+\mathcal{O} \left( \sqrt{\frac{-\sum_l \ln \det \R_l+\ln\frac{nm}{\delta}+(B+\epsilon)^2c^2\Phi(f_\w)}{\gamma^2m}} \right),
\end{aligned}
\end{equation}
\end{small}where 
\begin{small}
\begin{equation}\nonumber
\begin{aligned}
    \Phi(f_\w)=\prod_{l=1}^n \|\W_l\|_2^2 \sum_{l=1}^n \frac{\|\W_l\|^2_F}{\|\W_l\|^2_2} \Big(\sum_l\big(\|\R^c_l\|_2^{\frac{1}{2}}+\|\R^r_l\|_2^{\frac{1}{2}}\big)\Big)^{2}.
\end{aligned}
\end{equation}
\end{small}
\end{thm}
\begin{proof}
    See Appendix (Proof for \Cref{thm:advbound}). 
\end{proof}


\blue{
\begin{myrem}
The primary challenge in establishing \Cref{thm:advbound} lies in bounding the KL divergence for non-spherical random perturbations. 
We address this by developing bounds for both the correlated random perturbation and the perturbed model, with detailed derivations provided in the Appendix. 
The resulting theorem demonstrates that correlation matrices fundamentally influence the robust generalization bound through two key properties: their determinant and spectral norm.
\end{myrem}
}

In the following subsection, we analyze how correlation matrices estimated from both clean and adversarial data impact the robust generalization bound, establishing a direct connection between matrix norms and model robustness.

\subsection{Weight correlation matrices over clean data and adversarial data}
\label{sec:3.2}


In popular adversarial training methods, such as TRADES~\citep{zhang2019theoretically}, the optimization process for the classifier incorporates both clean and adversarial training data.
When considering an adversarially trained classifier, denoted as $f_\w(\cdot)$, we model the random perturbation $\ul$ as a combination of $\ul_\x$ (the perturbation relied on clean data) and $\ul_{\x'}$ (the perturbation relied on adversarial data). 
As a result, the true correlation matrix, $\R$, is similarly a blend of $\R_\x$ (the correlation matrix associated with clean data) and $\R_{\x'}$ (the correlation matrix associated with adversarial data).
\blue{In this subsection, we discuss the influence of $\R_\x$ and $\R_{\x'}$ to the developed robust generalization bound.}

\begin{mydef}
\label{def:correlation matrix}
Given the input $\x$, the classifier $f_\w(\cdot)$ and the adversarial noise $\delta^{adv}_\w(\x)$, let $\x' = \x+\delta^{adv}_\w(\x)$, $\ul_\x$ and $\ul_{\x'}$ be Gaussian distributed random vectors over $\x$ and $\x'$ with each element being identically distributed as $\mathcal{N}(0,\sigma^2)$ but not independent one another,
with corresponding correlation matrices as follows:
\begin{small}
\begin{equation} 
\begin{aligned}
\R_{\x} &:= \frac{1}{\sigma^2} \E_\x(\ul_\x \ul_\x^\top),\quad
\R_{\x'} &:= \frac{1}{\sigma^2}\E_{\x}(\ul_{\x'} \ul_{\x'}^\top).
\end{aligned}
\end{equation}
\end{small}
\end{mydef}

\begin{myass}
\label{ass:true correlation matrix}
Given an adversarially trained model $f_\w(\cdot)$ and the data distribution $\D$, assume that the true perturbation $\ul$ is a combination of two random variables $\ul_\x$ and $\ul_{\x'}$, i.e., $\ul = q\ul_{\x}+(1-q)\ul_{\x'}$ with the correlation matrix $\R = q\R_{\x}+(1-q)\R_{\x'}$, where $q\in [0,1]$.
\end{myass}
Following \Cref{def:correlation matrix}, let $\R_{l,\x}$ denote the weight correlation matrix of $l$-th layer over clean dataset,
we let
\begin{small}
\begin{equation}
\begin{aligned}
\Lambda_{l,\max}&=\max\big(\lambda_{\max}(\R_{l,\x}),\lambda_{\max}(\R_{l,\x'})\big),\\
\Lambda_{l,\min}&=\min\big(\lambda_{\min}(\R_{l,\x}),\lambda_{\min}(\R_{l,\x'})\big),\\
\Lambda^c_{l,\max}&=\max\big(\|\R^c_{l,\x}\|_2^{\frac{1}{2}},\|\R^c_{l,\x'}\|_2^{\frac{1}{2}}\big),\\
\Lambda^r_{l,\max}&=\max\big(\|\R^r_{l,\x}\|_2^{\frac{1}{2}},\|\R^r_{l,\x'}\|_2^{\frac{1}{2}}\big),
\end{aligned}
\end{equation}
\end{small}where $\lambda_{\max}(\cdot)$ and $\lambda_{\min}(\cdot)$ are the largest and the smallest singular value of the matrix, respectively. 
Note that $\R_{l,\x}$, $\R^c_{l,\x}$ and $\R^r_{l,\x}$ are symmetric positive semi-definite matrices, thus their singular values and eigenvalues coincide. 
Then, we get the following theorem.

\begin{thm}[Robust generalization bound over clean and adversarial data]
\label{thm:advbound2}
Given \Cref{thm:advbound}, let $\ul$ be a non-spherical Gaussian with the correlation matrices $\R$, $\R^c$ and $\R^r$ over $\x$ and $\x'$.
Then, for any $\delta,\gamma>0$, with probability at least $1-\delta$ over a training set of size $m$, for any $\w$, we have
\begin{small}
\begin{equation}\nonumber
\begin{aligned}
&\Lc^{adv}_{0}(f_\w) \! \le \! \widehat\Lc_{\gamma}^{adv}(f_\w)\\
&\quad+\mathcal{O} \left( \sqrt{\frac{(B+\epsilon)^2c^2\Phi(f_\w)-\sum_l\ln (\Lambda_{l,\min}^{k_l} \Lambda_{l,\max}^{\hh-k_l})+\ln\frac{nm}{\delta}}{\gamma^2m}} \right),
\end{aligned}
\end{equation}
\end{small}where 
\begin{small}
\begin{equation}\nonumber
\begin{aligned}
    \Phi(f_\w)=\prod_{l=1}^n \|\W_l\|_2^2 \sum_{l=1}^n \frac{\|\W_l\|^2_F}{\|\W_l\|^2_2} \Big(\sum_l\big(\Lambda^c_{l,\max}+\Lambda^r_{l,\max}\big)\Big)^{2},
\end{aligned}
\end{equation}
\end{small}$k_l=(\hh\Lambda_{l,\max}-\hh)/(\Lambda_{l,\max}-\Lambda_{l,\min})$, and $\Lambda_{l,\min}^{k_l} \Lambda_{l,\max}^{\hh-k_l}$ is the lower bound of $\det \R_l$.
\end{thm}
\begin{proof}
    See Appendix (Proof for \Cref{thm:advbound2}). 
\end{proof}


\blue{
\begin{myrem}
\Cref{thm:advbound2} establishes a refined robust generalization bound that incorporates non-spherical random perturbation of weights over both clean and adversarial data. 
The theorem reveals that simultaneously minimizing the spectral norm and maximizing the determinant of correlation matrices yields a tighter bound. 
These matrix norms, which we term second-order statistics of weights, emerge as crucial factors in determining classifier robustness during adversarial training.
\end{myrem}
}

\section{Estimation and optimization}
\label{sec:4}

\begin{figure}[t!]
\includegraphics[width=0.49
\textwidth]{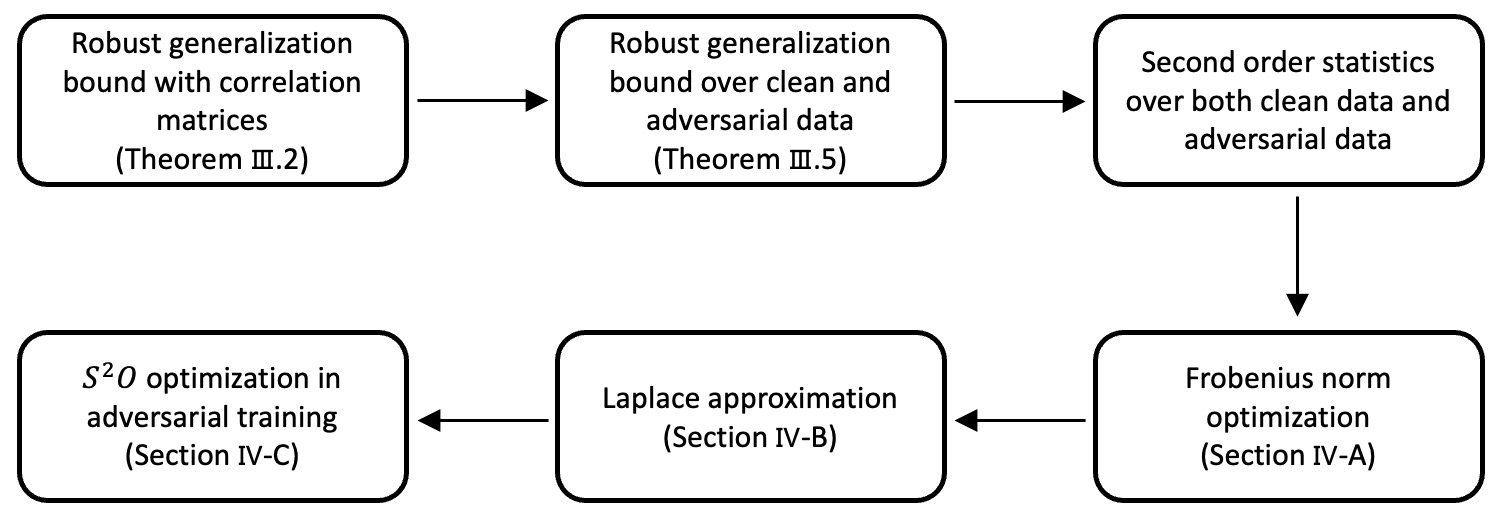}
\centering
\vspace{-5mm}
\caption{
\blue{Illustration of the optimization framework: 
\Cref{thm:advbound} and \Cref{thm:advbound2} show the influence of the second-order statistics of weights, over clean and adversarial data, on the robust generalization performance. 
In \Cref{sec:fnorm}, we demonstrate that these statistics can be approximately optimized using the Frobenius norm of the correlation matrix. 
To simplify this optimization process, we employ the Laplace approximation in \Cref{sec:laplace} and propose our adversarial training methodology, S$^2$O, in \Cref{sec:optimize}.}
}   
\vspace{-1mm}
\label{fig:method_flow}
\end{figure}

\blue{The structure of this section is delineated in \Cref{fig:method_flow}. 
Motivated by the theoretical insights from \Cref{thm:advbound} and \Cref{thm:advbound2}, we propose that second-order statistics of weights play a crucial role in determining DNN robustness. 
In \Cref{sec:fnorm}, we demonstrate that these statistics can be effectively captured and optimized through the Frobenius norm of the correlation matrix, $\|\R_l\|_F$. 
To enable efficient optimization, we employ Laplace approximation in \Cref{sec:laplace} to estimate the correlation matrix using Hessian. 
These developments culminate in our proposed adversarial training method, S$^2$O, presented in \Cref{sec:optimize}.}

\subsection{Frobenius norm optimization}
\label{sec:fnorm}
The adversarial training of a neural network is seen as a process of optimizing over an adversarial objective function. 
\Cref{thm:advbound} and \Cref{thm:advbound2} inspire us that the correlation matrix affects the robust generalization performance of neural networks, thus we apply the second-order statistics penalty terms, i.e., $-\ln \Lambda_{l,\min}^{k_l} \Lambda_{l,\max}^{\hh-k_l}$, $\Lambda^c_{l,\max}$, and $\Lambda^r_{l,\max}$, to the objective function $J_{\adv}$, and denote the new objective function as $\tilde{J}_{\adv}$.

Estimating and computing the gradient for the above norms in each iteration presents a formidable challenge due to its computational complexity. 
To address this, we propose approximating $\nabla_{\w_l}(-\ln\Lambda_{l,\min}^{k_l} \Lambda_{l,\max}^{\hh-k_l}+\Lambda^c_{l,\max}+\Lambda^r_{l,\max})$ through 
\begin{small}
\begin{equation}
    \nabla_{\w_l}(\g(\R_{l}))=\frac{\partial\big(\|\R_{l,\x}\|_F^2+\|\R_{l,\x'}\|_F^2\big)}{\partial \w_l},
\label{eq:R}
\end{equation}
\end{small}thereby reducing the algorithmic complexity.

While establishing a precise relationship between $\|\R_{l,\x}\|_F^2$, $\|\R_{l,\x'}\|_F^2$ and the above penalty terms is challenging, they are generally related. 
Specifically, when $\R_{l,\x}$ and $\R_{l,\x'}$ have uniform off-diagonal elements denoted by $r_\x$ and $r_{\x'}$, and $r_\x r_{\x'} \ge 0$, we observe the following correlations.
\begin{mypers}
\label{lem:4.1}
Decreasing $\|\R_{l,\x}\|_F^2$ and $\|\R_{l,\x'}\|_F^2$ leads to a decline in $|r_\x|$ and $|r_{\x'}|$, causing reduced $\Lambda^c_{l,\max}$ and $\Lambda^r_{l,\max}$.    
\end{mypers}
\begin{mypers}
\label{lem:4.2}
Decreasing $\|\R_{l,\x}\|_F^2$ and $\|\R_{l,\x'}\|_F^2$ leads to an increase in $\Lambda_{l,\min}^{k_l} \Lambda_{l,\max}^{\hh-k_l}$.
\end{mypers}
\begin{proof}
    See Appendix (Proof for Propositions). 
\end{proof}

In addition to the above theoretical insights, we extend our analysis by conducting simulations in \Cref{fig:s}, to explore the relationship between $\|\R_{l,\x}\|_F^2$, $\|\R_{l,\x'}\|_F^2$ and 
$\Lambda_{l,\min}^{k_l} \Lambda_{l,\max}^{\hh-k_l}$, $\Lambda^c_{l,\max}$, $\Lambda^r_{l,\max}$. 
The findings from \Cref{lem:4.1}, \Cref{lem:4.2}, along with the results of the simulations in \Cref{fig:s}, suggest that regularizing the term $\g(\R_{l})$ can effectively decrease $\Lambda^c_{l,\max}$ and $\Lambda^r_{l,\max}$, while simultaneously increasing $\Lambda_{l,\min}^{k_l} \Lambda_{l,\max}^{\hh-k_l}$. 
This provides a practical approach to optimize these matrix norms, thereby contributing to our understanding of their interplay.


While the direct optimization of $\g(\R_{l})$ still poses significant computational challenges, the application of the Laplace approximation in \Cref{sec:laplace} offers a viable solution. 
This approximation method substantially simplifies the computation of $\g(\R_{l})$, making the optimization process more feasible and efficient.

\begin{figure}[t!]
\includegraphics[width=0.49
\textwidth]{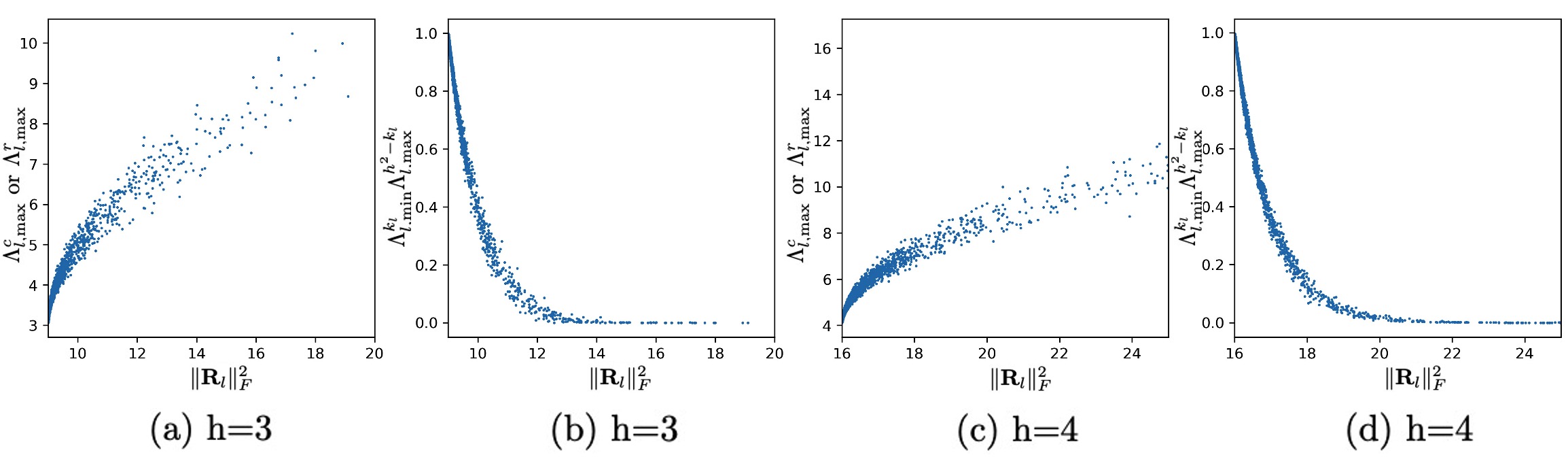}
\centering
\vspace{-7mm}
\caption{We sample 10000 9-dimensional correlation matrices and demonstrate \textbf{(a)} $\|\R_{l,\x}\|_F^2$, $\|\R_{l,\x'}\|_F^2$ w.r.t $\Lambda^c_{l,\max}$ or $\Lambda^r_{l,\max}$; 
\textbf{(b)} $\|\R_{l,\x}\|_F^2$, $\|\R_{l,\x'}\|_F^2$ w.r.t $\Lambda_{l,\min}^{k_l} \Lambda_{l,\max}^{\hh-k_l}$; 
\textbf{(c)} $\|\R_{l,\x}\|_F^2$, $\|\R_{l,\x'}\|_F^2$ w.r.t $\Lambda^c_{l,\max}$ or $\Lambda^r_{l,\max}$; 
\textbf{(d)} $\|\R_{l,\x}\|_F^2$, $\|\R_{l,\x'}\|_F^2$ w.r.t $\Lambda_{l,\min}^{k_l} \Lambda_{l,\max}^{\hh-k_l}$.
Note that the horizontal axis represents both $\|\R_{l,\x}\|_F^2$ and $\|\R_{l,\x'}\|_F^2$.
}   
\vspace{-1mm}
\label{fig:s}
\end{figure}

\subsection{Laplace approximation}
\label{sec:laplace}


The Laplace approximation is a widely used estimation technique within the Bayesian framework, employed for approximating posterior densities or moments, as discussed in \citet{ritter2018scalable}. 
This method approximates the posterior (e.g., $\w + \ul$) using a Gaussian distribution. 
It achieves this by employing the second-order Taylor expansion of the natural logarithm of the posterior, centered around its MAP estimate. 
To illustrate, for the weights at layer $l$ with an MAP estimate, denoted as $\w_l^*$ on $\St$, the $\ln$ probability of $\ln \Pro(\w_l+\ul_l|\St)$ can be approximately made up of the first Taylor term, $\ln \Pro(\w_l^*|\St)$, and the second Taylor term, $-\frac{1}{2}(\w_l-\w_l^*+\ul_l)^\top {\E_{\x}[\HH_{l}]} (\w_l-\w_l^*+\ul_l)$,
where $\E_{\x}[\HH_{l}]$ is the expectation of the Hessian matrix over input data $\x$, 
and the Hessian matrix $\HH_{l}$  
is given by
$\HH_{l}=\frac{\partial^2 \ell(f_{\w}(\x),y)}{\partial \w_l\partial \w_l}$.


Note that the first-order term of the Taylor polynomial, $\ln \Pro(\w_l^*|\St)$, is dropped as the gradient around the MAP estimate $\w_l^*$ is zero. 
Upon further examination, the second component of the expansion aligns precisely with the logarithm of the probability density function for a Gaussian-distributed multivariate random variable, characterized by mean $\w^*_l$ and covariance ${\Sigma_{l}}=\E_{\x}[\HH_{l}]^{-1}$. 
Consequently, it follows that $\w_l + \ul_l \sim \mathcal{N}(\w_l, \Sigma_{l})$. 
Here, $\Sigma_{l}$ can be viewed as the covariance matrix of $\ul_l$, and the learned weights $\w_l$ can be seen as the MAP estimate $\w_l^*$.

The Laplace approximation suggests that the efficient estimation of $\Sigma_{l}$ is feasible through the inverse of the Hessian matrix, given that ${\Sigma^{-1}_{l}} = \E_{\x}[\HH_{l}]$. 
Note that the term of $\Sigma^{-1}_{l,\x'} = \E_{\x'}[\HH_{l}]$ is omitted here due to its similarity to $\Sigma^{-1}_{l,\x} = \E_{\x}[\HH_{l}]$. 
Furthering this line of research, \citet{botev2017practical} and \citet{ritter2018scalable} have developed a Kronecker-factored Laplace approximation, drawing upon insights from second-order optimization techniques in neural networks. 
This approach, in contrast to traditional second-order methods in \citet{battiti1992first} and \citet{shepherd2012second} with high computational costs for deep neural networks, posits that the Hessian matrix of the $l$-th layer can be efficiently approximated using a Kronecker product, i.e.,

\vspace{-2mm}
\begin{small}
\begin{equation} \label{eq:hessian-decompose}
\begin{aligned}
    \HH_{l}=\underbrace{\ac_{l-1}\ac_{l-1}^\top}_{\mathcal{A}_{l-1}} \otimes \underbrace{\frac{\partial^2 \ell(f_{\w}(\x),y)}{\partial \h_{l}\partial \h_{l}}}_{\mathcal{H}_{l}}=\mathcal{A}_{l-1} \otimes \mathcal{H}_{l},
\end{aligned}
\vspace{-2mm}
\end{equation}
\end{small}where $\mathcal{A}_{l-1} \in \mathbb{R}^{h \times h}$ is the covariance of the post-activation of the previous layer, and $\mathcal{H}_{l} \in \mathbb{R}^{h \times h}$ is the Hessian matrix of the loss with respect to the pre-activation of the current layer, and $h$ is the number of neurons for each layer. 


To calculate $\E_{\x}[\HH_{l}]$, it is necessary to compute and aggregate the block diagonal matrix for each data sample. 
However, as the expectation of a Kronecker product is not always Kronecker factored, one would need to explicitly store the whole matrix $\HH_l$ to carry out this accumulation. 
With $h$ representing the dimensionality of a layer, the matrix would contain $\mathcal{O}(h^4)$ elements. 
For layers where $h$ is approximately $1000$, storing $\HH_l$ would require several terabytes of memory, rendering it impractical. 
Consequently, there is a pressing need for a diagonal block approximation that is not only computationally but also spatially efficient. 
To address this, we employ a factorized approximation as proposed by \citet{botev2017practical} and \citet{ritter2018scalable}:
\begin{small}
\begin{equation}
\label{eq11}
\begin{aligned}
    \E_{\x}[\HH_{l}] = \E_{\x}[\mathcal{A}_{l-1} \otimes \mathcal{H}_{l}]\approx \E_{\x}[\mathcal{A}_{l-1}] \otimes \E_{\x}[\mathcal{H}_{l}].
\end{aligned}
\end{equation}
\end{small}With this factorization, estimating the inverse of the Hessian for each layer becomes efficient through solving a Kronecker product form linear system, i.e.,
\begin{small}
\begin{equation}
\label{eq12}
\begin{aligned}
    \E_{\x}[\HH_{l}]^{-1} \approx \E_{\x}[\mathcal{A}_{l-1}]^{-1} \otimes \E_{\x}[\mathcal{H}_{l}]^{-1}.
\end{aligned}
\end{equation}
\end{small}

\subsection{Second-order statistics optimization (S$^2$O)}
\label{sec:optimize}


\blue{
Given \eqref{eq12}, we approximate the correlation matrix using the Hessian, where the correlation matrix over clean data comprises the right-hand two terms in \eqref{eq12}. 
A similar approximation applies to the correlation matrix over adversarial data. 
In this work, we focus on $\E_{\x}[\mathcal{A}_{l-1}]^{-1}$ and optimize $\|\R_{l,\x}\|_F^2$ by minimizing $\|\A_{l-1}\|_F^2$, where $\A_{l-1}$ is the normalized form of $\E_{\x}[\mathcal{A}_{l-1}]^{-1}$, i.e., for $\forall  0<i,j\le h$ and $i,j \in \N$,}
\begin{small}
\begin{equation}
\label{eq:2000}
    (\A_{l-1,\x})_{[ij]}=\frac{(\E_{\x}[\mathcal{A}_{l-1}]^{-1})_{[ij]}}{\sqrt{(\E_{\x}[\mathcal{A}_{l-1}]^{-1})_{[ii]}(\E_{\x}[\mathcal{A}_{l-1}]^{-1})_{[jj]}}}.
\end{equation}
\end{small}Then, we enhance the adversarial training by incorporating the regularization terms $\|\A_\x\|_F^2$ and $\|\A_{\x'}\|_F^2$, obtaining the new adversarial training objective function $\tilde{J}_{\adv}$ with
\begin{small}
\begin{equation}
\begin{aligned}
\nabla_{\w}\tilde{J}_{\adv}&=\nabla_{\w}J_{\adv}+\alpha(\nabla_{\w}\|\A_\x\|_F^2+\nabla_{\w}\|\A_{\x'}\|_F^2)\\
&=\nabla_{\w}J_{\adv}+\alpha\nabla_{\w}\g(\A).
\end{aligned}
\end{equation}
\end{small}Here, $\alpha \in [0,\infty)$ is a hyper-parameter to balance the relative contributions of the second-order statistics penalty term and the original objective function $J_{\adv}$. 
For computational efficiency, we only use the last layer to compute $\g(\A)$ in the experiments.

\begin{figure}[t!]
\includegraphics[width=0.5
\textwidth]{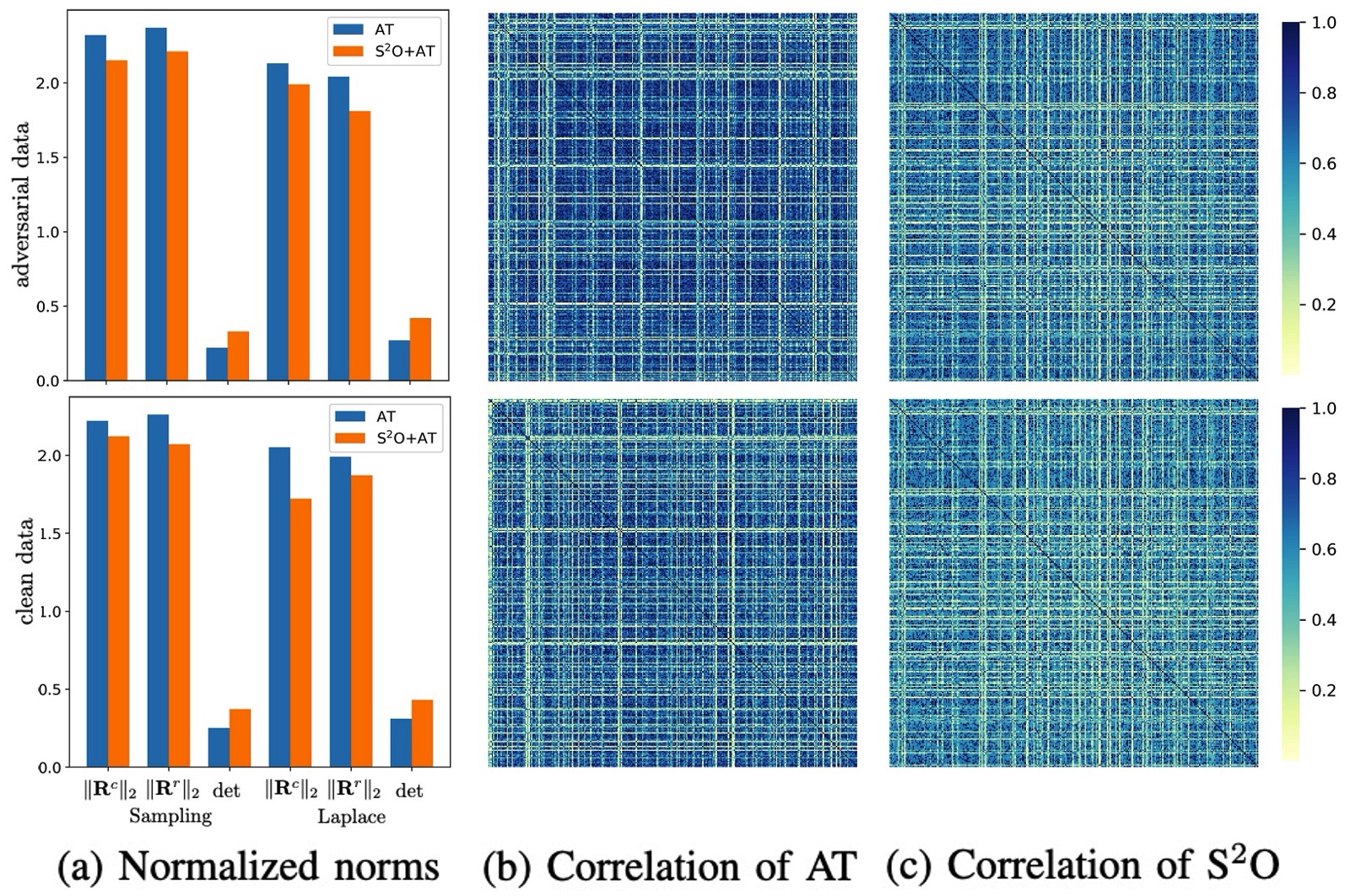}
\centering
\vspace{-7mm}
\caption{We train PreAct ResNet18 with AT and AT+S$^2$O on CIFAR-10, and show the results of partial weights. 
The results estimated from adversarial data are presented in the \textbf{top figures}, while the \textbf{bottom figures} show the results estimated from clean data.
\textbf{(a)} shows the normalized spectral norm of $\R^c$, ${\R}^r$, and the determinant of $\R$, with sampling estimation and Laplace approximation respectively. 
\textbf{(b)} and \textbf{(c)} demonstrate the absolute correlation matrix of partial weights, for AT and AT+S$^2$O respectively.}  
\vspace{-1mm}
\label{fig:2}
\end{figure}

\section{Experiments}
\label{sec:experiments}

\blue{In this section, we present our experimental evaluation, beginning with general experimental settings, while some specific configurations are detailed in their respective subsections. 
We first provide a comprehensive analysis of our second-order statistics optimization (S$^2$O) training method, demonstrating its effect to the second-order norms. 
We then assess the robustness of different methods against both white-box and black-box adversarial attacks across various benchmark datasets.}

\blue{
\emph{Datasets.} 
We evaluate our method on CIFAR-10/100~\cite{krizhevsky2009learning}, SVHN~\cite{netzer2011reading}, Tiny-ImageNet~\cite{deng2009imagenet}, and  Imagenette datasets. CIFAR-10 and CIFAR-100 contain $32\times 32$ pixel images with 10 and 100 categories respectively, SVHN consists of $32\times 32$ pixel digit images with 10 classes, and Tiny-ImageNet contains $64\times 64$ pixel images across 200 categories.
Imagenette is a subset of 10 classes from ImageNet-1K.
}

\blue{
\emph{Model architectures.} 
We evaluate different methods using PreAct ResNet-18 \citep{he2016deep} and WideResNet-34-10/28-10 \citep{zagoruyko2016wide} as primary architectures. 
We also conduct experiments  on ViT-B \cite{dosovitskiy2020image} and DeiT-S~\cite{touvron2021training} to further demonstrate the effectiveness of S$^2$O across architectures.
}

\blue{
\emph{Basic training setting.}
For the experiments of ResNet-18 and WideResNet training from scratch on CIFAR-10/100, the training duration spans 200 epochs using SGD with a momentum of 0.9, a batch size of 128, weight decay of $5 \times 10^{-4}$, and an initial learning rate of 0.1, which is reduced by a factor of 10 after 100 and 150 epochs. 
For the SVHN dataset, we maintain identical parameters except for an initial learning rate of 0.01. 
Additionally, simple data augmentation techniques, such as $32 \times 32$ random cropping with 4-pixel padding and random horizontal flipping, are utilized. 
In our default settings for S$^2$O, we use $\alpha=0.3$, except in the case of AT+S$^2$O (CIFAR-10) as specified in \Cref{tab:1}, where $\alpha=0.1$ is used. 
During adversarial training, we use PGD-10 with step size $2/255$ to craft $\ell_\infty$ adversarial examples and step size $15/255$ to craft $\ell_2$ adversarial examples.}

\blue{
\emph{Evaluation setting.}
In our experimental evaluation, we execute white-box attacks on models trained with both baseline methods and our S$^2$O-enhanced variants. 
These attacks include the FGSM \cite{goodfellow2014explaining}, PGD-20 \cite{madry2017towards}, and CW-20~\cite{carlini2017towards} (specifically, the $\ell_\infty$ version of the CW loss optimized using PGD-20).
We follow the commonly-used $\ell_\infty$ radius $8/255$ and $\ell_2$ radius $128/255$.
We also evaluate robustness using standard Auto Attack (AA) \cite{croce2020reliable}, an ensemble attack methodology that combines three white-box attacks (APGD-CE, APGD-DLR \cite{croce2020reliable}, and FAB \cite{croce2020minimally}) with a black-box attack (Square Attack \cite{andriushchenko2020square}). 
}

\blue{
\emph{Baselines.}
We introduce key baseline methods in this experimental study.
TRADES~\citep{zhang2019theoretically} establishes bounds on the gap between robust and clean error, motivating an adversarial training objective:
\begin{small}
\begin{equation}
\label{eq_trades}
    \mathop{\E}\limits_{(\x,y)\sim \St}\Big[ \ell(f_\w(\x),y)+\max_{\|\x'-\x\|_p\le \epsilon}\KL\big(f_\w (\x)\|f_\w (\x')\big)/\lambda\Big].
\end{equation}
\end{small}AWP~\citep{wu2020adversarial} extends the inner maximization by incorporating weight perturbations based on adversarial examples, optimizing a double-perturbation objective:
\begin{small}
\begin{equation}
\label{eq_awp}
\max_{\vec{v}\in\V}\mathop{\E}\limits_{(\x,y)\sim \St}\max_{\|\x'-\x\|_p\le \epsilon}\ell( f_{\w+\vec{v}}(\x'), y),
\end{equation}
\end{small}where $\V$ is a feasible region for the model weights perturbation. 
LBGAT~\citep{cui2021learnable} enhances robustness by replacing standard cross-entropy with logit margin regularization, maximizing the gap between the highest (correct class) and second-highest logits. 
Recent advances include leveraging DDPM-generated data for improved adversarial training~\citep{DBLP:journals/corr/abs-2202-10103,wang2023better}, employing stop-gradient operations to manage gradient conflicts~\citep{waseda2024rethinking} to improve the clean-robust trade-off (sum of clean accuracy and robust accuracy), and aligning logit distributions across different $\ell_p$ norm attacks~\citep{jiangramp} to improve multiple $\ell_p$ robustness.
}

\begin{table}[t]
\centering
\caption{Adversarial training across datasets on PreAct ResNet18 (\%). 
}
\label{tab:1}
\vspace{-2mm}
\renewcommand\arraystretch{1.35}
\scalebox{0.69}{
\begin{tabular}{clccccccccccc}
\specialrule{.1em}{.075em}{.075em}
Threat &\multirow{2}{*}{Method} & \multicolumn{3}{c}{CIFAR-10} && \multicolumn{3}{c}{CIFAR-100} && \multicolumn{3}{c}{SVHN} \\
\cline{3-5} \cline{7-9} \cline{11-13}
Model&& Clean & PGD & Time & & Clean & PGD & Time & & Clean & PGD & Time         
\\ \hline
\multirow{2}{*}{$\ell_\infty $} &AT & 82.41 & 52.77 & 309s & & 58.02 & 28.02 & 307s & & 93.17 & 60.91 & 509s \\
&\grr AT+S$^2$O & \grr \textbf{83.65} & \grr \textbf{55.11} & \grr 368s &\grr  & \grr \textbf{58.45} & \grr \textbf{30.58} & \grr 371s & \grr & \grr \textbf{93.39} &\grr \textbf{64.83} &\grr  595s  \\ 
\hline
\multirow{2}{*}{$\ell_2 $} &AT & 88.83 & 68.83 & 292s & & 64.21 & 42.20 & 290s & & 94.02 & 66.76 & 477s \\
&\grr AT+S$^2$O & \grr \textbf{89.57} & \grr \textbf{69.42} & \grr 364s &\grr  & \grr \textbf{65.32} & \grr \textbf{44.07} & \grr 366s & \grr  & \grr \textbf{94.93} & \grr \textbf{76.19} & \grr 586s \\ 
\specialrule{.1em}{.075em}{.075em}
\end{tabular}}
\vspace{2mm}
\centering
\caption{TRADES and AWP on CIFAR-10/100 with $\ell_\infty$ attack (\%).}
\label{tab:2}
\vspace{-2mm}
\renewcommand\arraystretch{1.35}
\resizebox{\linewidth}{!}{
\begin{tabular}{cclcccccc}
\specialrule{.1em}{.075em}{.075em}
Data/Net && \multicolumn{1}{c}{Method} && Clean & FGSM & PGD-20 & CW-20 & AA         \\ \cline{1-1} \cline{3-3}  \cline{5-9} 
\multirow{4}{*}{\shortstack{\scriptsize CIFAR-10\\ \scriptsize ResNet-18}} && TRADES && 82.89 & 58.72 & 53.81 & 51.83 & 48.6 \\
&&TRADES+AWP  && 82.30 & 59.48  & 56.18 & 53.12 & 51.7 \\
&&\grr TRADES+S$^2$O &\grr & \grr \textbf{84.15} & \grr 60.19 & \grr 55.20 & \grr 52.47 & \grr 49.5 \\
&&\grr TRADES+AWP+S$^2$O &\grr & \grr 83.79 & \grr \textbf{60.27} & \grr \textbf{57.29} & \grr \textbf{53.84} & \grr \textbf{52.4} \\
\cline{1-1} \cline{3-3}  \cline{5-9} 
\multirow{4}{*}{\shortstack{\scriptsize CIFAR-10\\ \scriptsize WRN}} && TRADES && 83.98 & 61.08 & 56.82 & 54.54 & 52.7 \\
&&TRADES+AWP  && 84.99 & 63.11 & 59.67 & 57.41 & \textbf{56.2}\\
&&\grr TRADES+S$^2$O &\grr & \grr 85.67 & \grr 62.73 & \grr 58.34 & \grr 55.36 & \grr 54.1\\
&&\grr TRADES+AWP+S$^2$O &\grr & \grr \textbf{86.01} & \grr \textbf{64.16} & \grr \textbf{61.12} & \grr \textbf{57.93} & \grr 55.9\\
\cline{1-1} \cline{3-3}  \cline{5-9} 
\multirow{5}{*}{\shortstack{\scriptsize CIFAR-100\\ \scriptsize WRN}} &&TRADES && 60.38 & 35.01 & 32.11 & 28.93 & 26.9 \\
&&TRADES+LBGAT  && 60.43 & - & 35.50 & \textbf{31.50} & 29.3 \\
&&TRADES+AWP && 60.27 & 36.12 & 34.04 & 30.64 & 28.5 \\
&& TRADES+S$^2$O && 63.40 & 35.96 & 33.06 & 29.57 & 27.6  \\
&& TRADES+AWP+S$^2$O && \textbf{64.17} & \textbf{37.98} & \textbf{35.95} & 31.26 & \textbf{29.9} \\
\specialrule{.1em}{.075em}{.075em}
\end{tabular}}
\vspace{-2mm}
\end{table}

\begin{table}[t!]
\centering
\caption{Experiments with DDPM generated data under $\ell_\infty$ attack (\%).}
\label{tab:ddpm}
\vspace{-1mm}
\renewcommand\arraystretch{1.35}
\scalebox{0.77}{
\begin{tabular}{cclcccc}
\specialrule{.1em}{.075em}{.075em} 
Dataset & & \multicolumn{1}{c}{Method} & Model & DDPM Data & Clean & AA \\ \hline
\multirow{12}{*}{\shortstack{CIFAR-10}} & & Lee et al. (2020) \cite{lee2020adversarial} & WRN-34-10 & - & 92.56 & 39.70 \\
& & Wang et al. (2020) \cite{wang2019improving} & WRN-34-10 & - & 83.51 & 51.10 \\
& & Rice et al. (2020) \cite{rice2020overfitting} & WRN-34-20 & - & 85.34 & 53.42 \\
& & Zhang et al. (2020) \cite{zhang2020attacks} & WRN-34-10 & - & 84.52 & 53.51 \\
& & Pang et al. (2021) \cite{pang2020bag} & WRN-34-20 & - & 86.43 & 54.39 \\
& & Jin et al. (2023) \cite{jin2023randomized} & WRN-34-20 & - & 85.98 & 54.20 \\
& & Gowal et al. (2020) \cite{DBLP:journals/corr/abs-2010-03593} & WRN-70-16 & - & 85.29 & 57.20 \\
& & Zhang et al. (2019) \cite{zhang2019theoretically} & WRN-34-10 & - & 84.65 & 53.04 \\
& & Wu et al. (2020) \cite{wu2020adversarial} & WRN-34-10 & - & 85.17 & 56.21 \\
& & Rebuffi et al. (2021) \cite{DBLP:journals/corr/abs-2103-01946} & WRN-28-10 & 1M & 85.97 & 60.73 \\
& & Pang et al. (2022) \cite{DBLP:journals/corr/abs-2202-10103} & WRN-28-10 & 1M & 88.61 & 61.04 \\
& & \quad + S$^2$O & \grr WRN-28-10 & \grr 1M & \grr \textbf{88.93} & \grr \textbf{61.45}  \\
\hline
\multirow{7}{*}{\shortstack{CIFAR-100}} & & Cui et al. (2021) \cite{cui2021learnable} & WRN-34-10 & - & 60.43 & 29.34 \\
& & Gowal et al. (2020) \cite{DBLP:journals/corr/abs-2010-03593} & WRN-70-16 & - & 60.86 & 30.03 \\
& & Zhang et al. (2019) \cite{zhang2019theoretically} & WRN-34-10 & - & 60.22 & 26.91  \\
& & Wu et al. (2020) \cite{wu2020adversarial} & WRN-34-10 & - & 60.38 & 28.63 \\
& & Rebuffi et al. (2021) \cite{DBLP:journals/corr/abs-2103-01946} & WRN-28-10 & 1M & 59.18 & 30.81 \\
& & Pang et al. (2022) \cite{DBLP:journals/corr/abs-2202-10103} & WRN-28-10 & 1M & 63.66 & 31.08 \\
& & \quad + S$^2$O & \grr WRN-28-10 & \grr 1M & \grr \textbf{63.77} & \grr \textbf{31.29}  \\
\hline
\multirow{3}{*}{\shortstack{Tiny-\\ImageNet}} & & Cui et al. (2021) \cite{DBLP:conf/nips/GowalRWSCM21} & WRN-28-10 & 1M &  60.95 & 26.66 \\
& & Wang et al. (2023) \cite{wang2023better} & WRN-28-10 & 1M & 65.19 & 31.30 \\
&& \quad + S$^2$O &  WRN-28-10 & 1M & \textbf{65.28} & \textbf{31.60}  \\ 
\specialrule{.1em}{.075em}{.075em} 
\end{tabular}}
\vspace{0mm}
\end{table}

\subsection{Second-order statistics}
\label{sec:5.1}

In this section, we analyze how our S$^2$O method affects the second-order statistics of weights, specifically examining changes in the weight correlation matrix.

We employ two complementary methods to minimize estimation errors: Laplace approximation (detailed in \Cref{sec:4}) and sampling method. 
The sampling method estimates second-order statistics using multiple weight samples of the form $\w+\ul$, where $\ul$ represents zero-mean Gaussian noise constrained by $|\Lc(f_{\w+\ul})-\Lc(f_{\w})|\le \epsilon'$ ($\epsilon'=0.05$ for CIFAR-10). 
This follows established sharpness-based methodologies \cite{keskar2016large,jiang2019fantastic}. 
Valid weight samples are obtained by training the noise-perturbed converged network for 50 epochs with a learning rate of 0.0001.

The experimental results in \Cref{fig:2}(a) demonstrate that S$^2$O effectively reduces the spectral norms of both clean ($\R^c$) and robust (${\R}^r$) correlation matrices while increasing the determinant of $\R$. Additionally, \Cref{fig:2}(b) and (c) reveal the ability of S$^2$O to reduce weight correlations. 
These findings demonstrate how S$^2$O systematically optimizes the statistical properties of weight matrices, directly impacting network performance.

\subsection{Applying S$^2$O on vanilla adversarial training}
\label{sec:5.2}

Initially, we deploy PreAct ResNet-18 to assess the efficacy of our S$^2$O method, integrated with vanilla PGD-10 adversarial training. 
This evaluation, accommodating both $\ell_{\infty}$ and $\ell_2$ threat models, spans a range of datasets, including CIFAR-10, CIFAR-100, and SVHN.

\blue{\Cref{tab:1} shows that S$^2$O-enhanced variants improve both clean accuracy and robustness against PGD-20 attacks across all datasets. 
Under the $\ell_{\infty}$ threat model on CIFAR-10, AT+S$^2$O achieves 2\%-3\% higher robust accuracy against PGD-20 attacks and 1\%-1.5\% higher clean accuracy compared to vanilla adversarial training. The computational overhead of S$^2$O is modest, increasing training time by approximately 20\% per epoch. These improvements remain consistent across datasets and attack scenarios while maintaining reasonable computational costs.
}

\blue{
\emph{Time efficiency.} As discussed above, S$^2$O introduces a modest increase in training time, primarily due to the computation of post-activation correlation matrices defined in \eqref{eq:2000} for both clean and adversarial data in each minibatch. 
To investigate potential efficiency improvements, we conducted experiments on CIFAR-10 ($\ell_\infty$) for ResNet-18 using correlation matrices computed only from adversarial data. 
This modification reduces training time from 368s/epoch to 340s/epoch, eliminating approximately half of the additional computational overhead (309s/epoch for AT). 
However, this efficiency gain comes with slight performance degradation: clean accuracy decreases from 83.65\% to 82.89\%, and PGD-20 accuracy from 55.11\% to 55.02\%. 
This demonstrates a clear trade-off between computational efficiency and model performance, allowing us to balance these factors based on their specific requirements.
}



\begin{table}[t!]
\centering
\caption{Experiments on CIFAR-10 with multiple $\ell_p$ attacks (\%).}
\label{tab:5}
\vspace{-2mm}
\renewcommand\arraystretch{1.35}
\scalebox{0.8}{
\begin{tabular}{lllcccccc}
\specialrule{.1em}{.075em}{.075em} 
\multicolumn{1}{c}{Radius} && \multicolumn{1}{c}{Method} && Clean & $\ell_\infty$ & $\ell_2$ & $\ell_1$ & Union
\\ \cline{1-1} \cline{3-3} \cline{5-9}
\multirow{4}{*}{\shortstack{\scriptsize $\ell_\infty: \frac{2}{255}$\\ \scriptsize $\ell_2: 0.5$\\ \scriptsize $\ell_1: 12$}} && Croce et al. (2022) \cite{croce2022adversarial} && 87.2 & 73.3 & 64.1 & 55.4 & 55.4    \\
&& Tramer et al. (2019) \cite{tramer2019adversarial} && 85.6 & 72.1 & 63.6 & 56.4 & 56.4    \\
&& Jiang et al. (2024) \cite{jiangramp} && 86.3 & 73.3 & 64.9 & 59.1 & 59.1    \\
&& \quad + S$^2$O && 87.1 & 73.6 & 65.3 & 59.9 & \textbf{59.9}    \\
\cline{1-1} \cline{3-3} \cline{5-9}
\multirow{4}{*}{\shortstack{\scriptsize $\ell_\infty: \frac{8}{255}$\\ \scriptsize $\ell_2: 1.5$\\ \scriptsize $\ell_1: 12$}} && Croce et al. (2022) \cite{croce2022adversarial} && 83.5 & 41.0 & 25.5 & 52.9 & 25.5    \\
&& Tramer et al. (2019) \cite{tramer2019adversarial} && 74.6 & 42.9 & 35.7 & 50.3 & 35.6    \\
&& Jiang et al. (2024) \cite{jiangramp} && 74.4 & 43.4 & 37.2 & 51.1 & 37.1    \\
&& \quad + S$^2$O && 74.9 & 43.4 & 37.6 & 52.6 & \textbf{37.5}     \\
\specialrule{.1em}{.075em}{.075em} 
\end{tabular}
}
\vspace{-2mm}
\end{table}
\begin{table}[t!]
\centering
\caption{Experiments on CIFAR to measure the trade-off ($\ell_\infty$, \%).}
\label{tab:tradeoff}
\vspace{-2mm}
\renewcommand\arraystretch{1.35}
\scalebox{0.8}{
\begin{tabular}{llccccccc}
\specialrule{.1em}{.075em}{.075em} 
\multicolumn{1}{c}{\multirow{2}{*}{Method}} && \multicolumn{3}{c}{CIFAR-10} && \multicolumn{3}{c}{CIFAR-100}   \\
&& Clean & AA & Sum && Clean & AA & Sum          \\ \cline{1-1} \cline{3-5} \cline{7-9} 
AT &&  86.06 & 46.26 & 132.32 && 59.83 & 23.94 & 83.77   \\
Suzuki et al. (2023) \cite{suzuki2023adversarial} && 90.24 & 50.20 & 140.44 && 73.05 & 24.32 & 97.37   \\
Waseda et al. (2025) \cite{waseda2024rethinking} &&  90.89 & 50.77 & 141.66 && 72.51 & 24.18 & 96.70   \\
\quad +S$^2$O && 90.45 & 53.16 & \textbf{143.61} && 72.36 & 27.09 & \textbf{99.45}    \\
\specialrule{.1em}{.075em}{.075em} 
\end{tabular}
}
\vspace{0mm}
\end{table}

\subsection{Compare with popular methods}

In the subsequent experiments, we utilize both PreAct ResNet-18 and WideResNet-34-10 to assess the efficacy of our S$^2$O method in conjunction with two state-of-the-art methods, TRADES and TRADES+AWP. 
These evaluations are conducted on CIFAR-10/100 under an $\ell_{\infty}$ threat model. 
The robustness of all defense models is rigorously tested against a suite of white-box attacks, including FGSM, PGD-20, CW-20, as well as the comprehensive Auto Attack.

As shown in \Cref{tab:2}, the S$^2$O-enhanced variants consistently outperform the existing models, with only one exception noted.
While the AWP method shows an improvement over TRADES, the integration of S$^2$O offers further enhancements. 
For instance, the S$^2$O-enhanced TRADES+AWP model on WideResNet exhibits a 1.45\% higher accuracy against PGD-20 attacks compared to the TRADES+AWP model. 
Similarly, the S$^2$O-enhanced TRADES model on PreAct ResNet-18 demonstrates a 1.39\% increase in accuracy against PGD-20 attacks compared to the TRADES-only model.
Additional experiments conducted on CIFAR-100, as shown in \Cref{tab:2}, reveal that S$^2$O also improves robustness compared to other methods, like TRADES+LBGAT~\cite{cui2021learnable}, under most attack scenarios.

\blue{
Moreover, we evaluate the effectiveness of S$^2$O when combined with DDPM-generated data for adversarial training. 
Based on the pre-trained WideResNet-28-10 models from \citet{DBLP:journals/corr/abs-2202-10103} and \citet{wang2023better}, we fine-tune the models for 5 epochs using 1 million DDPM-generated samples on CIFAR and Tiny-ImageNet datasets.
We incorporate S$^2$O ($\alpha=0.3$) into their training framework, maintaining original settings except for adopting a learning rate of $0.001$. 
Results in \Cref{tab:ddpm} demonstrate that S$^2$O enhances both clean and adversarial accuracy across all datasets, validating its effectiveness in DDPM-based adversarial training.
}



\begin{table}[t!]
\centering
\caption{Experiments on adversarially trained ViT, DeiT ($\ell_\infty$, \%).}
\label{tab:3}
\vspace{-2mm}
\renewcommand\arraystretch{1.35}
\scalebox{0.9}{
\begin{tabular}{clccccccc}
\specialrule{.1em}{.075em}{.075em} 
\multicolumn{1}{c}{\multirow{2}{*}{Network}} && \multicolumn{1}{c}{\multirow{2}{*}{Method}} && \multicolumn{2}{c}{CIFAR-10} && \multicolumn{2}{c}{Imagenette}          \\
&& && Clean & AA && Clean & AA          \\ \cline{1-1} \cline{3-3} \cline{5-6} \cline{8-9} 
\multirow{2}{*}{\shortstack{\scriptsize ViT-B}} && AT && 85.53 & 48.33 && 91.40 & 64.20   \\
&& AT+S$^2$O && 85.66 & \textbf{48.97} && 91.40 & \textbf{66.40}  \\
\cline{1-1} \cline{3-3} \cline{5-6} \cline{8-9} 
\multirow{2}{*}{\shortstack{\scriptsize DeiT-S}} && AT && 84.13 & 47.26 && 90.40 & 60.20   \\
&& AT+S$^2$O && 84.42 & \textbf{48.36} && 91.00 & \textbf{63.60}  \\
\specialrule{.1em}{.075em}{.075em} 
\end{tabular}
}
\vspace{3mm}
\centering
\caption{Sensitivity analysis on CIFAR-10 for ResNet-18 ($\ell_\infty$, \%).}
\label{tab:6}
\vspace{-2mm}
\renewcommand\arraystretch{1.35}
\scalebox{0.83}{
\begin{tabular}{lccccccc}
\specialrule{.1em}{.075em}{.075em} 
\multicolumn{1}{c}{Method} && Clean & AA && PGD-20$_{train}$ & PGD-20$_{test}$ & Gap 
\\ \cline{1-1} \cline{3-4} \cline{6-8}
AT && 82.41 & 47.1 && 62.33 & 52.77 & 9.56  \\
AT+S$^2$O (0.05) && 83.22 & \textbf{48.5} && 61.99 & 53.82 & 8.17  \\
AT+S$^2$O (0.1) && \textbf{83.65} & 48.3 && 61.50 & \textbf{55.11} & 6.39   \\
AT+S$^2$O (0.2) && 83.43 & 47.8 && 60.27 & 54.59 & 5.68  \\ 
AT+S$^2$O (0.3) && 82.89 & 46.5 && 59.36 & 54.24 & 5.12   \\
AT+S$^2$O (0.4) && 82.54 & 46.7 && 58.41 & 52.92 & 5.49   \\ 
\specialrule{.1em}{.075em}{.075em} 
\end{tabular}}
\end{table}

\subsection{Experiments under multiple $\ell_p$ attacks}

\blue{
Following the baseline configuration of \citet{jiangramp}, we train PreAct ResNet18 on CIFAR-10 with S$^2$O regularization ($\alpha=0.3$) under their framework. 
We evaluate model performance using clean accuracy, Auto Attack robustness across ${\ell_1, \ell_2, \ell_\infty}$ perturbations, and union accuracy (simultaneous robustness against all three norm attacks). 
The union accuracy, by definition, cannot exceed the lowest individual norm accuracy. 
As shown in \Cref{tab:5}, combining the framework of \citet{jiangramp} with S$^2$O achieves state-of-the-art union accuracy at two evaluation radii: ${\ell_\infty: \frac{2}{255}, \ell_2: 0.5, \ell_1: 12}$ and ${\ell_\infty: \frac{8}{255}, \ell_2: 1.5, \ell_1: 12}$. 
These findings demonstrate the ability of S$^2$O to enhance multi-norm robustness in existing defense methods.
}

\subsection{Experiments on robustness-accuracy trade-off}

\blue{
A fundamental challenge in adversarial training is the trade-off between robustness and accuracy, which recent works quantify as the sum of clean and robust accuracy \cite{suzuki2023adversarial,waseda2024rethinking}. 
Invariance regularization addresses this trade-off by promoting model stability under adversarial perturbations. 
Building on the invariance framework of \citet{waseda2024rethinking}, we train WideResNet-34-10 on CIFAR-10/100 by combining S$^2$O ($\alpha=0.3$) with their AR-AT method. 
Results in \Cref{tab:tradeoff} demonstrate that this combination surpasses existing approaches on both datasets in terms of the clean-robust accuracy sum.
}

\subsection{Experiments on ViTs}

\blue{
In further investigations, we utilize ViTs to evaluate the performance of our S$^2$O method in conjunction with vanilla adversarial training on CIFAR-10 and Imagenette under $\ell_{\infty}$ threat model. 
All ViT-B and DeiT-S models are pre-trained on ImageNet-1K and are adversarially trained for 40 epochs using SGD with weight decay $1\times 10^{-4}$, gradient clipping, and an initial learning rate of 0.1, decreased by a factor of 10 at epochs 36 and 38. 
Simple data augmentations such as random crop with padding and random horizontal flip are applied. 
To accommodate the smaller image size of CIFAR-10, we downsample the patch embedding kernel from $16 \times 16$ to $4 \times 4$.
The robustness of all defense models, including these S$^2$O-enhanced variants, is rigorously assessed against standard Auto Attack.

As shown in \Cref{tab:3}, S$^2$O improves both clean and robust accuracy of ViT-B and DeiT-S models against Auto Attack. 
The improvements are particularly pronounced on Imagenette, where S$^2$O enhances robust accuracy by 2.2\% for ViT-B and 3.4\% for DeiT-S.}

\subsection{Sensitivity analysis, black-box and BPDA evaluation}

In our study, we conduct a hyper-parameter sensitivity analysis, specifically examining the influence of $\alpha$ for ResNet-18 on CIFAR-10, as shown in \Cref{tab:6}.

We evaluate the effectiveness of S$^2$O under black-box (transfer) attacks by integrating it with vanilla adversarial training on ResNet-18 architectures. 
Using an $\ell_{\infty}$ threat model, we conduct experiments across CIFAR-10/100 and SVHN datasets, where transfer attacks are generated from identically trained adversarial models. 
Results in \Cref{tab:7} show that S$^2$O consistently enhances model robustness against transfer attacks.

\blue{Backwards Pass Differentiable Approximation (BPDA) \citep{athalye2018obfuscated} addresses the limitation of standard gradient-based attacks when defense mechanisms incorporate non-differentiable or gradient-masking operations. 
BPDA circumvents these obstacles by approximating the gradient of defense function using a differentiable surrogate function. 
We evaluate ResNet-18 models on CIFAR-10/100 using $\ell_\infty$ norm BPDA attacks with 20 iterations. Results in \Cref{tab:8} demonstrate that S$^2$O enhances model robustness even under BPDA evaluation.}

\begin{table}[t!]
\centering
\caption{Experiments on ResNet-18 with transfer attacks ($\ell_\infty$, \%).}
\label{tab:7}
\vspace{-2mm}
\renewcommand\arraystretch{1.35}
\scalebox{0.82}{
\begin{tabular}{lccccccccc}
\specialrule{.1em}{.075em}{.075em} 
\multirow{2}{*}{Method} && \multicolumn{2}{c}{CIFAR-10} && \multicolumn{2}{c}{CIFAR-100} && \multicolumn{2}{c}{SVHN}  \\
&& FGSM & PGD-20 && FGSM & PGD-20 && FGSM & PGD-20 \\ 
\cline{1-1} \cline{3-4} \cline{6-7} \cline{9-10}
AT && 64.32 & 62.63 && 38.55 & 37.36 && 71.77 & 63.76  \\
\grr AT+S$^2$O &\grr & \grr \textbf{65.63} & \grr \textbf{63.87} &\grr & \grr \textbf{39.68} & \grr \textbf{38.60} &\grr & \grr \textbf{72.20} & \grr \textbf{64.31} \\
\specialrule{.1em}{.075em}{.075em} 
\end{tabular}}
\vspace{3mm}
\centering
\caption{Experiments on ResNet-18 with BPDA attack ($\ell_\infty$, \%).}
\label{tab:8}
\vspace{-2mm}
\renewcommand\arraystretch{1.35}
\scalebox{0.82}{
\begin{tabular}{lcccccc}
\specialrule{.1em}{.075em}{.075em} 
\multirow{2}{*}{Method} && \multicolumn{2}{c}{CIFAR-10} && \multicolumn{2}{c}{CIFAR-100}  \\
&& Clean & BPDA && Clean & BPDA \\ 
\cline{1-1} \cline{3-4} \cline{6-7}
AT && 82.41 & 75.62 && 58.02 & 51.37  \\
AT+S$^2$O & & \textbf{83.65} & \textbf{78.87} && \textbf{58.45} & \textbf{54.01}  \\
\specialrule{.1em}{.075em}{.075em} 
\end{tabular}}
\vspace{-2mm}
\end{table}

\section{Conclusion}

This work studies a previously neglected aspect in the field of adversarial training, advocating for the systematic consideration of the second-order statistics of neural network weights.
By integrating theoretical advancements (updating the PAC-Bayesian framework), algorithmic innovations (efficient estimation of weight correlation matrix, effective adversarial training using S$^2$O), and comprehensive experimental results, our study demonstrates that incorporating second-order statistics of weights enhances both robustness and generalization performance, outperforming both vanilla adversarial training and state-of-the-art adversarial training methods.



{\appendix[Proof for Theorem~\ref{thm:advbound}]
In the following, we firstly present a set of fundamental definitions of margin operator and robust margin operator, and key theoretical results of \Cref{lem:3.3}, \Cref{lem:A1}, and \Cref{lem:A2}, drawing from \citet{xiao2023pac}.
Then, we introduce \cref{lem:3.4}. 
Following the theoretical framework developed in \citet{neyshabur2017pac,xiao2023pac}, we provide the proof which is structured into two primary steps. 
In the initial stage, we leverage \Cref{lem:A2} and \Cref{lem:3.4} to compute the maximal permissible perturbation $\ul$, ensuring compliance with the given condition on the margin $\gamma$. 
In the second phase, we focus on calculating the $\KL$ component of the bound, taking into account the previous perturbation. 
This calculation is pivotal for the derivation of the PAC-Bayesian bound.

\blue{First, we introduce the concept of a local perturbation bound \cite{xiao2023pac}. 
To illustrate this, we consider a scalar value function $g_\w(\x): \mathcal{X}_{B,d} \to \mathbb{R}$. 
This function could manifest in various forms, such as the $i$-th output of a neural network, represented by $f_\w(\x)[i]$, the margin operator defined by $f_\w(\x)[y] - \max\limits_{j \neq y} f_\w(\x)[j]$, or the robust margin operator. 
Each of these representations serves to elucidate different aspects of the output of the neural network in relation to the input $\x$ within the specified domain.

\begin{mydef}[Local perturbation bound]
\label{def:Local perturbation bound}
Given $\mathbf{x} \in \mathcal{X}_{B,d}$, we say $g_{\mathbf{w}}(\mathbf{x})$ has a $\left(L_1, \cdots, L_n\right)$-local perturbation bound w.r.t. $\w$, if
\begin{small}
\begin{equation}
\left|g_{\mathbf{w}}(\mathbf{x})-g_{\mathbf{w}^{\prime}}(\mathbf{x})\right| \leq \sum_{l=1}^n L_l\left\|\W_l-\W_l^{\prime}\right\|_2,
\end{equation}
\end{small}where $L_l$ can be related to $\mathbf{w}, \mathbf{w}^{\prime}$ and $\mathbf{x}$.
\end{mydef}

The bound in \Cref{def:Local perturbation bound} is pivotal in quantifying the variation in the output of the function $g_\w(\x)$, particularly in response to minor perturbations in the weights of DNNs. 
Building upon this foundation, we get the following Lemma, as detailed in the work of \citet{xiao2023pac}.

\begin{mylem}[Bound for the perturbed model, \citet{xiao2023pac}]
\label{lem:3.3}
If $g_{\mathbf{w}}(\mathbf{x})$ has a $\left(A_1|\mathbf{x}|, \cdots, A_n|\mathbf{x}|\right)$-local perturbation bound, i.e.,
\begin{small}
\begin{equation}
\left|g_{\mathbf{w}}(\mathbf{x})-g_{\mathbf{w}^{\prime}}(\mathbf{x})\right| \leq \sum_{l=1}^n A_l|\mathbf{x}|\left\|\W_l-\W_l^{\prime}\right\|_2,
\end{equation}
\end{small}the robustified function $\mathop{\max}\limits_{\left\|\mathbf{x}-\mathbf{x}^{\prime}\right\|_2 \leq \epsilon} g_{\mathbf{w}}\left(\mathbf{x}^{\prime}\right)$ has a $(A_1(|\mathbf{x}|+\epsilon),$$\cdots, A_n(|\mathbf{x}|+\epsilon))$-local perturbation bound.
\end{mylem}

\Cref{lem:3.3} elucidates that the local perturbation bound of the robustified function, denoted as $\mathop{\max}\limits_{\|\mathbf{x}-\mathbf{x}^{\prime}\|_2 \leq \epsilon} g_{\mathbf{w}}(\mathbf{x}^{\prime})$, can be effectively approximated by the local perturbation bound of the function $g_\w(\x)$. 
This insight is fundamental in establishing the robust generalization bound in \Cref{thm:2.4}.}

\vspace{3mm}

\noindent \emph{Margin Operator}: Following the notation employed by \citet{bartlett2017spectrally,xiao2023pac}, the margin operator is defined for the true label $y$ given an input $\x$, as well as for a pair of classes $(i, j)$. 
This definition provides a clear and precise measure of class separation for the model.
\begin{equation}
\begin{aligned}
&M(f_\w(\mathbf{x}), y) = f_\w(\mathbf{x})[y] - \max_{i \neq y} f_\w(\mathbf{x})[i],\\
&M(f_\w(\mathbf{x}), i, j) = f_\w(\mathbf{x})[i] - f_\w(\mathbf{x})[j].
\end{aligned}
\end{equation}

\noindent \emph{Robust Margin Operator}: 
Analogously, the robust margin operator is also defined for a pair of classes $(i, j)$ with respect to $(\x,y)$. 
\begin{small}
\begin{equation}
\begin{aligned}
&RM(f_{\w}(\mathbf{x}), y) = \max_{\|\mathbf{x} - \mathbf{x'}\| \leq \epsilon} \left(f_{\w}(\mathbf{x'})[y] - \max_{j \neq y} f_{\w}(\mathbf{x'})[j]\right), \\
&RM(f_{\w}(\mathbf{x}), i, j) = \max_{\|\mathbf{x} - \mathbf{x'}\| \le \epsilon} \left(f_{\w}(\mathbf{x'})[i] - f_{\w}(\mathbf{x'})[j]\right).
\end{aligned}
\end{equation}
\end{small}

\noindent Based on the above definitions and \Cref{lem:3.3}, \citet{xiao2023pac} provide the form of $A_i$ for the margin operator through the following lemma.

\begin{mylem}[\citet{xiao2023pac}]
\label{lem:A1}
Consider $f_{\mathbf{w}}(\cdot)$ as an $n$-layer neural network characterized by ReLU activation functions. 

\vspace{2mm}

\noindent 1. Given $\x$ and $i, j$, the margin operator $M(f_\w(\x), i, j)$ has a $(A_1|\x|, · · · , A_n|\x|)$-local perturbation bound w.r.t. $\w$, where $A_l = 2e\prod^n_{l=1} \|\W_l\|_2/ \|\W_l\|_2$. And
\begin{small}
\begin{equation}
\begin{aligned}
&\left| M(f_{\mathbf{w}+\mathbf{u}}(\mathbf{x}), i, j) - M(f_{\mathbf{w}}(\mathbf{x}), i, j) \right| \\
&\quad\quad\;\quad\quad\quad\quad\quad\quad\quad\quad\quad\;\leq 2eB \prod_{l=1}^{n} \|\W_l\|_2 \sum_{l=1}^{n} \frac{\|\U_l\|_2}{\|\W_l\|_2}.
\end{aligned}
\end{equation}
\end{small}

\noindent 2. Given $\x$ and $i, j$, the robust margin operator $RM(f_\w(\x), i, j)$ has a $(A_1(|\x|+\epsilon), · · · , A_n(|\x|+\epsilon))$-local perturbation bound w.r.t. $\w$. And
\begin{small}
\begin{equation}
\begin{aligned}
&\left| RM(f_{\mathbf{w}+\mathbf{u}}(\mathbf{x}), i, j) - RM(f_{\mathbf{w}}(\mathbf{x}), i, j) \right| \\
&\quad\quad\quad\quad\quad\quad\quad\quad\quad\leq 2e(B+\epsilon) \prod_{l=1}^{n} \|\W_l\|_2 \sum_{l=1}^{n} \frac{\|\U_l\|_2}{\|\W_l\|_2}.
\end{aligned}
\end{equation}
\end{small}
\end{mylem}
\noindent Drawing upon \Cref{lem:A1}, \citet{xiao2023pac} develop the following lemma, which demonstrates that we can leverage the
weight perturbation of the robust margin operator to develop PAC-Bayesian framework. 
\begin{mylem}[\citet{xiao2023pac}]
\label{lem:A2}
Let $f_{\w}: \X \to \Y$ be any predictor with parameters $\w$, and P be any distribution on the parameters that is independent of the training data. 
Then, for any $\gamma, \delta>0$, with probability at least $1-\delta$ over the training set of size $m$, for any $\w$, and any random perturbation $\ul$ s.t.

\noindent Case 1. 
\begin{small}
\begin{equation}\nonumber
\begin{aligned}
\mathbb{P}_\ul\Bigg[\mathop{\max}_{i, j \in [n^y] \atop \x \in \X} |M(f_{\w+\ul}(\x), i, j) &- M(f_{\w}(\x), i, j)| \le \frac{\gamma}{2}\Bigg] \geq \frac{1}{2},
\end{aligned}
\end{equation}
\end{small}we have
\begin{small}
\begin{equation}\nonumber
\Lc_0(f_\w) \le \widehat{\Lc}_{\gamma}(f_\w) + 4\sqrt{\frac{\KL(Q_{\w + \ul}\|P)}{m - 1} + \ln \frac{6m}{\delta}}.
\end{equation}
\end{small}\noindent Case 2. 
\begin{small}
\begin{equation}\nonumber
\begin{aligned}
\mathbb{P}_\ul\Bigg[\mathop{\max}_{i, j \in [n^y]\atop \x \in \X} |RM(f_{\w+\ul}(\x), i, j) &- RM(f_{\w}(\x), i, j)| \le \frac{\gamma}{2}\Bigg] \geq \frac{1}{2},
\end{aligned}
\end{equation} 
\end{small}we have
\begin{small}
\begin{equation}\nonumber
\Lc_0^{adv}(f_\w) \le \widehat{\Lc}^{adv}_{\gamma}(f_\w) + 4\sqrt{\frac{\KL(Q_{\w + \ul}\|P)}{m - 1} + \ln \frac{6m}{\delta}}.
\end{equation}
\end{small}
\end{mylem}

\blue{
In our research, we integrate correlated relationships into the random perturbation $\ul$. 
This integration results in a more intricate bound for $\ul$, reflecting the complexities introduced by these correlations.

\begin{mypers}[Bound for the random perturbation]
\label{lem:3.4}
Consider the random perturbation matrix $\U_l$ with correlation matrices $\R_l^c$ and $\R_l^r$, with probability at least $\frac{1}{2}$, we have 
\begin{small}
\begin{equation}
\|\U_l\|_2 \le c\Big ( \|\R^c_l\|_2^{\frac{1}{2}}+\|\R^r_l\|_2^{\frac{1}{2}} \Big)\sigma.
\end{equation}
\end{small}Here, $c > 0$ represents a universal constant, which is related to factors such as $h$, $n$, and various other relevant parameters.
\end{mypers}
\begin{proof}
    See Appendix (Proof for Propositions). 
\end{proof}


Note that the probability of $\frac{1}{2}$ serves as a default value, as adopted in \citet{neyshabur2017pac}. 
Drawing from the foundation laid out in \Cref{lem:3.3}, which establishes the bound for the perturbed model, and \Cref{lem:3.4}, which details the bound for random perturbation, we are able to formulate the following PAC-Bayesian robust generalization bound, details are given in the following. 
}

\vspace{3mm}

\noindent \emph{Proof for \Cref{thm:advbound}.} Consider a neural network with weight matrices denoted as $\W$. 
To regularize the network, each weight matrix $\W_l$ can be scaled by its spectral norm $\|\W_l\|_2$.
Define $\beta$ as the geometric mean of these spectral norms across all weight matrices, specifically, $\beta = \left(\prod_{l=1}^n \|\W_l\|_2\right)^{1/n}$. 
We then introduce a modified version of the weight matrices, represented as $\widetilde{\W}_l = \frac{\beta}{\|\W_l\|_2} \W_l$, achieved by scaling the original weights $\W_l$ by the factor $\frac{\beta}{\|\W_l\|_2}$. 
Owing to the homogeneity property of the ReLU function, the network behavior using these modified weights, expressed as $f_{\widetilde{\w}}(\cdot)$, remains consistent with that of the original network $f_{\w}(\cdot)$.

Furthermore, this analysis reveals that the product of the spectral norms of the original weight matrices, $\left(\prod_{l=1}^n \|\W_l\|_2\right)$, is equivalent to the product of the spectral norms of the modified weights, $\left(\prod_{l=1}^n \|\widetilde{\W}_l\|_2\right)$. 
In addition, there exists an equivalence in the ratio of the Frobenius norm to the spectral norm between the original and modified weights, i.e., $\frac{\|\W_l\|_F}{\|\W_l\|_2} = \frac{\|\widetilde{\W}_l\|_F}{\|\widetilde{\W}_l\|_2}$. 
This observation implies that the excess error remains invariant under the application of this weight normalization strategy. 
Therefore, it is sufficient to establish the theorem for the case of the normalized weights $\widetilde{\w}$ alone. 
Following previous work, we assume that the spectral norm of each weight matrix is uniformly equal to $\beta$, that is, $\|\W_l\|_2 = \beta$ for any given layer $l$.

In this study, we set the prior distribution $P$ as a Gaussian distribution, characterized by a zero mean and a diagonal covariance matrix $\sigma^2 \mathbf{I}$. 
The parameter $\sigma$ will be subsequently determined by $\beta$. 
Considering that the prior must remain independent of the learned predictor $f_\w(\cdot)$ or its norm, we opt to set $\sigma$ based on an estimated value, $\tilde{\beta}$. 
We methodically compute the PAC-Bayesian bound for each value of $\tilde{\beta}$ selected from a pre-determined grid. 
This computation aims to establish a generalization guarantee for all $\w$ satisfying the condition $|\beta - \tilde{\beta}| \le \frac{1}{n} \beta$, thereby ensuring the inclusion of each pertinent $\beta$ by some $\tilde{\beta}$ in the grid. 
Subsequently, a union bound is applied across all values of $\tilde{\beta}$. 
For the current scope, we fix $\tilde{\beta}$ and consider $\w$ fulfilling $|\beta - \tilde{\beta}| \leq \frac{1}{n} \beta$, leading to the bounds $\frac{1}{e} \beta^{n-1} \leq \tilde{\beta}^{n-1} \leq e \beta^{n-1}$.

Recalling \Cref{lem:3.4}, with probability at least $\frac{1}{2}$, we have 
\begin{small}
\begin{equation}
\|\U_l\|_2 \le c\Big ( \|\R^c_l\|_2^{\frac{1}{2}}+\|\R^r_l\|_2^{\frac{1}{2}} \Big)\sigma,
\end{equation}
\end{small}where $c>0$ is a universal constant.
Plugging the bounds into \Cref{lem:A1} and \Cref{lem:A2}, we have that 
\begin{small}
\begin{equation}
\label{eq:a1}
\begin{aligned}
&\max_{i, j \in [n^y], \x \in \X} \left|RM(f_{\w+\ul}(\x), i, j) - RM(f_{\w}(\x), i, j)\right| \\
&\le 2e(B+\epsilon) \beta^{n-1} \sum_{l=1}^{n} \|\U_l\|_2\\
&\le 2e^2c\sigma(B+\epsilon) \tilde \beta^{n-1} \sum_{l=1}^{n} (\|\R^c_l\|_2^{\frac{1}{2}}+\|\R^r_l\|_2^{\frac{1}{2}})\\
&\le \frac{\gamma}{2}
\end{aligned}
\end{equation}
\end{small}To make \eqref{eq:a1} hold, given $\tilde{\beta}^{n-1} \leq e \beta^{n-1}$, we can choose the largest $\sigma$ as
\begin{small}
\begin{equation}
\sigma = \frac{\gamma}{4e^3c(B+\epsilon) \prod_{l=1}^{n} \|\W_l\|_2 \sum_{l=1}^{n} \frac{\|\R^c_l\|_2^{\frac{1}{2}}+\|\R^r_l\|_2^{\frac{1}{2}}}{\|\W_l\|_2}}.
\end{equation}
\end{small}Thus, we have
\begin{small}
\begin{equation}\nonumber
\label{eq:a3}
\begin{aligned}
\KL(Q_{\w+\ul}\|P)&=\sum_{l=1}^n \Big( \frac{\|\W_l\|_F^2}{2\sigma^2}-\ln \det \R_l \Big)\\
& =  \frac{16e^6c^2(B+\epsilon)^2 \left(\mathop{\sum}\limits_{l=1}^n\frac{\|\W_l\|_F^2}{\|\W_l\|_2^2}\right) (\mathop{\prod}\limits_{l=1}^{n} \|\W_l\|_2^2)}{\gamma^2}\\
&\quad\cdot (\mathop{\sum}\limits_{l=1}^{n} (\|\R^c_l\|_2^{\frac{1}{2}}+\|\R^r_l\|_2^{\frac{1}{2}}))^2-\sum_{l=1}^n\ln \det \R_l. 
\end{aligned}
\end{equation}
\end{small}

Then, we can present a union bound over different choices of $\tilde \beta$.
We only need to develop the bound for $\left(\frac{\gamma}{2 B}\right)^{\frac{1}{n}} \leq \beta \leq\left(\frac{\gamma \sqrt{m}}{2 B}\right)^{\frac{1}{n}}$ which can be covered with a size of $nm^{\frac{1}{2n}}$, as discussed in \citet{neyshabur2017pac}.
Integrating the above results with \Cref{lem:A2}, with probability at least $1-\delta$, for any $\tilde \beta$ and for all $\w$ such that $|\beta-\tilde{\beta}| \leq \frac{1}{n} \beta$, we have 
\begin{small}
\begin{equation}
\begin{aligned}
&\Lc^{adv}_{0}(f_\w) \le  \widehat\Lc_{\gamma}^{adv}(f_\w)\\
&+\mathcal{O} \left( \sqrt{\frac{-\sum_l \ln \det \R_l+\ln\frac{nm}{\delta}+(B+\epsilon)^2c^2\Phi(f_\w)}{\gamma^2m}} \right),
\end{aligned}
\end{equation}
\end{small}where 
\begin{small}
\begin{equation}\nonumber
\begin{aligned}
    \Phi(f_\w)=\prod_{l=1}^n \|\W_l\|_2^2 \sum_{l=1}^n \frac{\|\W_l\|^2_F}{\|\W_l\|^2_2} \Big(\sum_{l=1}^n\big(\|\R^c_l\|_2^{\frac{1}{2}}+\|\R^r_l\|_2^{\frac{1}{2}}\big)\Big)^{2}.
\end{aligned}
\end{equation}
\end{small}Hence, proved. \hfill $\square$
}

{
\appendix[Proof for Theorem~\ref{thm:advbound2}]
As mentioned in \Cref{ass:true correlation matrix}, we assume $\R_l$ is a combination of $\R_{l,\x}$ and $\R_{l,\x'}$ with an unknown coefficient $q$. 
We thus seek to bound $\|\R^c_l\|_2$, $\|\R^r_l\|_2$ and $\det \R_l$ through $\R_{l,\x}$ and $\R_{l,\x'}$, and build a bound based on $\R_{l,\x}$ and $\R_{l,\x'}$.

At first, $\|\R^c_l\|_2$ and $\|\R^r_l\|_2$ can be bounded through \citet{knutson2001honeycombs}, i.e.,
\begin{small}
\begin{equation}
\begin{aligned}
\|\R^c_l\|_2&\le q\lambda_{\max}(\R^c_{l,\x})+(1-q)\lambda_{\max}(\R^c_{l,\x'})\le (\Lambda^c_{l,\max})^2,
\end{aligned}
\end{equation}
\end{small}and similarly,
\begin{equation}
\begin{aligned}
\|\R^r_l\|_2&\le q\lambda_{\max}(\R^r_{l,\x})+(1-q)\lambda_{\max}(\R^r_{l,\x'})\le (\Lambda^r_{l,\max})^2.
\end{aligned}
\end{equation}
Then, we bound $\det \R_l$ through \citet{knutson2001honeycombs,kalantari2001determinantal}. 
Given any vector $\s$, we have 
\begin{small}
\begin{equation}
\begin{aligned}
    \langle \s, \R_l\s \rangle&=\langle \s, (q\R_{l,\x}+(1-q)\R_{l,\x'})\s \rangle\\
    &\ge (q\lambda_{\min}(\R_{l,\x})+(1-q)\lambda_{\min}(\R_{l,\x'}))\|\s\|_2^2\\
    &\ge \Lambda_{l,\min}\|\s\|_2^2.
\end{aligned}
\end{equation}
\end{small}Hence, we can get $\lambda_{\min}(\R_l)\ge\Lambda_{l,\min}$. 
According to \citet{knutson2001honeycombs} and the determinant lower bound in \citet{kalantari2001determinantal}, we have $\lambda_{\max}(\R_l)\le \Lambda_{l,\max}$, and
\begin{small}
\begin{equation}
    \det \R_l \ge \Lambda_{l,\min}^{k_l} \Lambda_{l,\max}^{\hh-k_l},
\end{equation}
\end{small}where $k_l=(\hh\Lambda_{l,\max}-\hh)/(\Lambda_{l,\max}-\Lambda_{l,\min})$.

Plugging the above results into \Cref{thm:advbound}, with probability at least $1-\delta$, we have
\begin{small}
\begin{equation}\nonumber
\begin{aligned}
&\Lc^{adv}_{0}(f_\w) \! \le \! \widehat\Lc_{\gamma}^{adv}(f_\w)\\
&\quad+\mathcal{O} \left( \sqrt{\frac{(B+\epsilon)^2c^2\Phi(f_\w)-\sum_l\ln (\Lambda_{l,\min}^{k_l} \Lambda_{l,\max}^{\hh-k_l})+\ln\frac{nm}{\delta}}{\gamma^2m}} \right),
\end{aligned}
\end{equation}
\end{small}where 
\begin{small}
\begin{equation}\nonumber
\begin{aligned}
    \Phi(f_\w)=\prod_{l=1}^n \|\W_l\|_2^2 \sum_{l=1}^n \frac{\|\W_l\|^2_F}{\|\W_l\|^2_2} \Big(\sum_l\big(\Lambda^c_{l,\max}+\Lambda^r_{l,\max}\big)\Big)^{2},
\end{aligned}
\end{equation}
\end{small}and $k_l=(\hh\Lambda_{l,\max}-\hh)/(\Lambda_{l,\max}-\Lambda_{l,\min})$.

\noindent Hence, proved. \hfill $\square$
}

{
\appendix[Proof for propositions]

\noindent \emph{Proof for \Cref{lem:3.4}.}
The work of \citet{bandeira2021spectral} implies that 
\begin{small}
\begin{equation}\nonumber
\begin{aligned}
   \E\|\U_l\|_2&\lesssim (1+\sqrt{\ln h})\|\E(\U_l^\top\U_l)\|_2^{\frac{1}{2}}+\|\E(\U_l\U_l^\top)\|_2^{\frac{1}{2}} \\
   &\le c_1\Big((1+\sqrt{\ln h})\|\E(\U_l^\top\U_l)\|_2^{\frac{1}{2}}+\|\E(\U_l\U_l^\top)\|_2^{\frac{1}{2}} \Big),
\end{aligned}
\end{equation}
\end{small}\begin{small}\begin{equation}\nonumber
    \Pro\Big(\Big| \|\U_l\|_2-\E\|\U_l\|_2 \Big|\ge t\Big)\le 2e^{-t^2/2\sigma_{*}(\U_l)^2},
\end{equation}
\end{small}\begin{small}\begin{equation}\nonumber
    \sigma_{*}(\U_l)\le \|\E(\U_l^\top\U_l)\|_2^{\frac{1}{2}},
\end{equation}
\end{small}where $c_1>0$ is a universal constant.
Taking a union bond over the layers, with probability at least $\frac{1}{2}$, the spectral norm of $\U_l$ is bounded by
\begin{small}
\begin{equation}
(\sqrt{2\ln(4n)}+c_1(1+\sqrt{\ln h}))\|\E(\U_l^\top\U_l)\|_2^{\frac{1}{2}}+c_1\|\E(\U_l\U_l^\top)\|_2^{\frac{1}{2}}.
\end{equation}
\end{small}Then, let $c=\sqrt{2\ln(4n)}+c_1 (1+\sqrt{\ln h})$, with probability at least $\frac{1}{2}$, we have
\begin{small}
\begin{equation}
\begin{aligned}
\|\U_l\|_2&\le (\sqrt{2\ln(4n)}+c_1(1+\sqrt{\ln h}))\|\E(\U_l^\top\U_l)\|_2^{\frac{1}{2}}\\
&\quad\quad\quad\quad\quad\quad\quad\quad\quad\quad\quad\quad\quad\quad\;+c_1\|\E(\U_l\U_l^\top)\|_2^{\frac{1}{2}}\\
&\le c\|\E(\U_l^\top\U_l)\|_2^{\frac{1}{2}}+c\|\E(\U_l\U_l^\top)\|_2^{\frac{1}{2}} \\
& = c\Big ( \|\R^c_l\|_2^{\frac{1}{2}}+\|\R^r_l\|_2^{\frac{1}{2}} \Big)\sigma.
\end{aligned}
\end{equation}
\end{small}Hence, proved. \hfill $\square$

$\quad$

\noindent \emph{Proof for \Cref{lem:4.1}.}
Given $\R_{l}\in \mathbb{R}^{h^2\times h^2}$, $\R^c_{l}\in \mathbb{R}^{h\times h}$, $\R^r_{l}\in \mathbb{R}^{h\times h}$, let $r_{\x}\ge 0$ and $ r_{\x'}\ge 0$, we have
\begin{small}
\begin{equation}
\begin{aligned}
\Lambda^c_{l,\max}&=\max\big(\|\R^c_{l,\x}\|_2^{\frac{1}{2}},\|\R^c_{l,\x'}\|_2^{\frac{1}{2}}\big)\\
&=\sqrt{h\big(1+(h-1)\max(r_\x,r_{\x'})\big)}
\end{aligned}    
\end{equation}
\end{small}and 
\begin{small}
\begin{equation}
\begin{aligned}
\Lambda^r_{l,\max}&=\max\big(\|\R^r_{l,\x}\|_2^{\frac{1}{2}},\|\R^r_{l,\x'}\|_2^{\frac{1}{2}}\big)\\
&=\sqrt{h\big(1+(h-1)\max(r_\x,r_{\x'})\big)}.
\end{aligned}    
\end{equation}
\end{small}Thus, decreasing $\|\R_{l,\x}\|_F^2$ and $\|\R_{l,\x'}\|_F^2$ leads to a decline in  $\Lambda^c_{l,\max}$ and $\Lambda^r_{l,\max}$.

\quad

\noindent Let $r_{\x}\le 0$ and $ r_{\x'}\le 0$, we have
\begin{small}
\begin{equation}
\begin{aligned}
\Lambda^c_{l,\max}&=\max\big(\|\R^c_{l,\x}\|_2^{\frac{1}{2}},\|\R^c_{l,\x'}\|_2^{\frac{1}{2}}\big)=\sqrt{h\big(1-\min(r_\x,r_{\x'})\big)}
\end{aligned}    
\end{equation}
\end{small}and 
\begin{small}
\begin{equation}
\begin{aligned}
\Lambda^r_{l,\max}&=\max\big(\|\R^r_{l,\x}\|_2^{\frac{1}{2}},\|\R^r_{l,\x'}\|_2^{\frac{1}{2}}\big)=\sqrt{h\big(1-\min(r_\x,r_{\x'})\big)}.
\end{aligned}    
\end{equation}
\end{small}Thus, decreasing $\|\R_{l,\x}\|_F^2$ and $\|\R_{l,\x'}\|_F^2$ leads to a decline in  $\Lambda^r_{l,\max}$ and $\Lambda^c_{l,\max}$.

\noindent Hence, proved. \hfill $\square$

$\quad$

\noindent \emph{Proof for \Cref{lem:4.2}.}
Given $\R_{l}\in \mathbb{R}^{h^2\times h^2}$, let $r_{\x}\ge r_{\x'}\ge 0$, we have
\begin{small}
\begin{equation}
\begin{aligned}
    c(r)&=\Lambda_{l,\min}^{k_l} \Lambda_{l,\max}^{\hh-k_l}=(1-r_\x)^{\hh-1}(1+(\hh-1)r_\x)\\
\end{aligned}
\end{equation}
\end{small}and 
\begin{small}
\begin{equation}
\begin{aligned}
    \frac{\partial c(r)}{\partial r_\x} = -\hh(\hh-1)r_\x (1-r_\x)^{\hh-2}\le 0,
\end{aligned}
\end{equation}
\end{small}it is obvious to get that $c(r)$ is negatively correlated with $r_\x$. 
Similarly, if $r_{\x'}\ge r_{\x}\ge 0$, we can get $c(r)$ is negatively correlated with $r_{\x'}$. 
Thus, decreasing $\|\R_{l,\x}\|_F^2$ and $\|\R_{l,\x'}\|_F^2$ leads to an increase in $\Lambda_{l,\min}^{k_l} \Lambda_{l,\max}^{\hh-k_l}$.  

\quad

\noindent Let $r_{\x}\le r_{\x'}\le 0$, we have
\begin{small}
\begin{equation}
\begin{aligned}
    c(r)&=\Lambda_{l,\min}^{k_l} \Lambda_{l,\max}^{\hh-k_l}=(1+(\hh-1)r_\x)(1-r_\x)^{\hh-1}\\
\end{aligned}
\end{equation}
\end{small}and
\begin{small}
\begin{equation}
\begin{aligned}
    \frac{\partial c(r)}{\partial r_\x} = -\hh(\hh-1)r_\x (1-r_\x)^{\hh-2}\ge 0,
\end{aligned}
\end{equation}
\end{small}it is clear to get that $c(r)$ is positively correlated with $r_\x$. Similarly, if $r_{\x'}\le r_{\x}\le 0$, we can get $c(r)$ is positively correlated with $r_{\x'}$. 
Thus, decreasing $\|\R_{l,\x}\|_F^2$ and $\|\R_{l,\x'}\|_F^2$ leads to an increase in $\Lambda_{l,\min}^{k_l} \Lambda_{l,\max}^{\hh-k_l}$.

\noindent Hence, proved. \hfill $\square$
}


\bibliographystyle{IEEEtranN}
\bibliography{sn-bibliography}

\vfill

\end{document}